\RequirePackage{fix-cm}
\documentclass{article}

\usepackage{iclr2027_conference,times}

\usepackage[utf8]{inputenc} 
\usepackage[T1]{fontenc}    
\usepackage{hyperref}       
\usepackage{url}            
\usepackage{amsfonts}       
\usepackage{nicefrac}       
\usepackage{microtype}      
\usepackage{xcolor}         
\newcommand{\lossrevision}[1]{#1}
\usepackage{comment}
\usepackage{amsmath, amssymb, amsthm}
\usepackage{algorithm}
\usepackage{algpseudocode}
\usepackage{booktabs}
\usepackage{multirow}
\usepackage{graphicx}
\usepackage{array}
\usepackage{tabularx}
\usepackage{subcaption} 

\usepackage{enumitem}
\makeatletter
\renewcommand{\section}{\@startsection{section}{1}{\z@}
  {-2.0ex plus -0.5ex minus -0.2ex}{1.0ex plus 0.2ex minus 0.2ex}
  {\large\sc\raggedright}}
\renewcommand{\subsection}{\@startsection{subsection}{2}{\z@}
  {-1.8ex plus -0.5ex minus -0.2ex}{0.5ex plus 0.1ex}
  {\normalsize\sc\raggedright}}
\renewcommand{\subsubsection}{\@startsection{subsubsection}{3}{\z@}
  {-1.5ex plus -0.5ex minus -0.2ex}{0.3ex plus 0.1ex}
  {\normalsize\sc\raggedright}}
\makeatother
\newtheorem{lemma}{Lemma}
\newtheorem{proposition}{Proposition}
\newtheorem*{outerproprestated}{Proposition~\ref{prop:outer_supergrad_main}}
\newtheorem{theorem}{Theorem}
\newtheorem{corollary}{Corollary}

\newcommand{\ie}{{\it i.e.}}

\newcommand{\bc}{\begin{center}}
	\newcommand{\ec}{\end{center}}
\newcommand{\be}{\begin{equation}}
	\newcommand{\ee}{\end{equation}}
\newcommand{\ba}{\begin{array}}
	\newcommand{\ea}{\end{array}}
\newcommand{\bean}{\setlength\arraycolsep{1pt}\begin{eqnarray*}}
	\newcommand{\eean}{\end{eqnarray*}}
\newcommand{\bea}{\setlength\arraycolsep{1pt}\begin{eqnarray}}
	\newcommand{\eea}{\end{eqnarray}}
\newcommand{\ben}{\begin{enumerate}}
	\newcommand{\een}{\end{enumerate}}
\newcommand{\bed}{\begin{itemize}}
	\newcommand{\eed}{\end{itemize}}

\DeclareMathOperator*{\argmin}{argmin}

\graphicspath{ {./images/} }
\usepackage{silence}
\title{Distributionally Robust Survival Models \\
under Subpopulation Shift and Outlier \\
Contamination}

\author{Seonghwi Kim \\
Pohang University of Science and Technology \\
\texttt{kshwi@postech.ac.kr}
\And
Sung Ho Jo \\
Pohang University of Science and Technology \\
\texttt{tjdgh1813@postech.ac.kr}
\And
Minwoo Chae\thanks{Corresponding author.} \\
Pohang University of Science and Technology \\
\texttt{mchae@postech.ac.kr}
}

\iclrfinalcopy
\hypersetup{
  hidelinks,
  pdftitle={Distributionally Robust Survival Models under Subpopulation Shift and Outlier Contamination},
  pdfauthor={Seonghwi Kim, Sung Ho Jo, Minwoo Chae}
}

\begin{document}

\maketitle
\lhead{Preprint}

\begin{abstract}
Learning robust survival models under distribution shift is an important but challenging problem in many applications. In heterogeneous populations, a model that performs well on average may still perform poorly on certain subpopulations, and this issue becomes even more severe when the training data are contaminated by outliers. In this paper, we propose a novel distributionally robust framework for survival analysis that jointly addresses latent subpopulation shift and outlier contamination. The proposed method combines an outer minimization that selects a refined nominal distribution by reducing the influence of contaminated samples and an inner maximization that focuses on the most challenging subpopulation. This formulation directly accommodates non-decomposable survival losses while preserving interactions across samples, including the risk-set structure of the Cox negative partial log-likelihood. We develop an alternating gradient-based algorithm with outer updates derived from the KKT conditions of the inner maximization. Experiments on simulated data and two survival benchmarks demonstrate that the proposed method remains robust when subpopulation shift and outlier contamination occur simultaneously. It stabilizes training in contaminated settings and substantially improves worst-group performance across both linear and nonlinear survival models, while maintaining competitive and sometimes superior overall performance.
\end{abstract}

\section{Introduction}
Survival analysis is a fundamental problem in statistics and machine learning, with broad applications in healthcare, reliability analysis, and customer retention. In many applications, the population comprises subpopulations with distinct survival patterns. Such heterogeneity may arise across demographic groups such as age, sex, and race, as well as across hospitals, clinical environments, or latent patient subtypes~\citep{hu2024fairness, rahman2022fair, fan2024stable}. When the mixture proportions of these subpopulations differ between training and test distributions, a model with strong average performance can still perform poorly on some subpopulations~\citep{sagawa2019distributionally, yang2023change, hu2024fairness}.

To address this challenge, distributionally robust learning, or distributionally robust optimization\footnote{While DRO was originally studied in the context of optimization, we use the term more broadly to refer to learning methods for robustness under distribution shift.} (DRO), has been widely studied as a principled framework for improving robustness under distributional shift~\citep{mohajerin2018data, duchi2021learning}.
Applying DRO to survival analysis, however, is not straightforward. \lossrevision{Many commonly used survival losses are non-decomposable, involving interactions across samples.} For example, the Cox negative partial log-likelihood depends on risk-set interactions, whereas ranking-based survival losses are defined through pairwise comparisons across samples~\citep{cox1972regression, steck2007ranking, lee2018deephit}. \lossrevision{In contrast, many existing DRO methods assume additive losses, which decompose into a sum of terms that each depend on a single sample.} \lossrevision{This structural difference prevents these methods from being directly applied to such survival losses}~\citep{jin2025distributionally, hu2022distributionally, hu2024fairness}.

Recently, \citet{hu2024fairness} proposed an $f$-divergence based DRO ($f$-DRO) framework for survival models, focusing on the $\chi^2$-divergence, and introduced a sample-splitting strategy to handle the interaction structure of survival losses. Their results show that $f$-DRO can improve robustness to latent subpopulation shift in survival analysis. However, this approach is highly sensitive to outliers or contaminated observations, since it improves robustness by assigning larger weights to high-loss samples. DORO~\citep{zhai2021doro} addresses both subpopulation shift and outlier contamination, but is designed for additive losses and does not directly apply to non-decomposable survival losses with cross-sample dependence.

To address these limitations, we propose a coupled distributionally robust framework for survival analysis that jointly addresses latent subpopulation shift and outlier contamination. The proposed formulation is based on a mixture-based ambiguity set with different roles in a coupled minimax problem.
The outer minimization refines the nominal distribution to reduce the influence of potentially contaminated samples. Based on this refined distribution, the inner maximization concentrates on the most challenging subpopulation. In this way, the proposed method iteratively refines the nominal distribution by reducing the influence of contaminated samples and reweights the most challenging subpopulation within the refined distribution, providing robustness to both outlier contamination and subpopulation shift.

For differentiable non-decomposable survival losses, we develop an efficient alternating gradient-based algorithm with outer updates derived from the KKT conditions of the inner maximization. We instantiate the framework with the Cox negative partial log-likelihood, preserving its risk-set interactions. We further show that, for additive losses, the formulation admits a two-stage structure: the outer minimization performs $\epsilon$-trimming, and the inner maximization reweights the highest-loss $\alpha$-fraction of the remaining probability mass.

Through simulation studies and experiments on two widely used survival benchmarks, METABRIC and FLC, we show that the proposed method remains robust under subpopulation shift and outlier contamination. In simulations, it exhibits stable training trajectories even in contaminated settings, whereas $f$-DRO exhibits unstable trajectories. On real benchmarks, it yields substantially improved worst-group performance across both linear and nonlinear survival models. It also achieves competitive overall performance, often even outperforming existing baselines in contaminated settings.

Our main contributions can be summarized as follows.
\begin{itemize}[left=1pt]
\item We propose a novel coupled distributionally robust framework for survival analysis that jointly addresses two practically important challenges: latent subpopulation shift and outlier contamination.
\item We develop an alternating gradient-based algorithm that directly optimizes \lossrevision{non-decomposable survival losses} while preserving interactions across samples. Its outer updates are derived from the KKT conditions of the inner maximization.
\item We demonstrate through simulation studies and real data experiments that the proposed method remains stable under outlier contamination and achieves substantially improved worst-group performance under subpopulation shift.
\end{itemize}

\vspace{-0.5em}
\section{Related Work}
\vspace{-0.2em}
\paragraph{Survival Analysis under Subpopulation Shift}
Survival analysis has traditionally been developed based on classical statistical frameworks, including nonparametric estimators, proportional hazards models, and related semiparametric approaches~\citep{cox1972regression, kaplan1958nonparametric, buckley1979linear, bennett1983analysis, kalbfleisch2002statistical}. More recently, it has been extended to machine learning methods that allow more flexible nonlinear modeling of survival outcomes~\citep{ishwaran2008random}, and further to deep learning based approaches designed for richer representation learning and more flexible risk prediction~\citep{lee2018deephit, ranganath2016deep, katzman2018deepsurv, avati2020countdown, tang2023survival}. However, these models do not explicitly account for subpopulation shift, and therefore often perform poorly on specific groups.

This issue has been studied in survival analysis mainly from two directions: fairness across observed groups and federated survival analysis under institutional heterogeneity. Fairness-based approaches typically impose constraints or regularization to encourage similar predictions or predictive performance across demographic or clinically defined groups~\citep{keya2021equitable, zhang2022longitudinal, zhang2023fairness, rahman2022fair, do2023fair}. Federated survival methods, in contrast, address heterogeneity across hospitals or institutions with different local patient populations and data distributions~\citep{andreux2020federated, archetti2023heterogeneous, seidi2024addressing, archetti2023federated, rahman2023fedpseudo, rahman2025fair, wang2022survmaximin}. However, these approaches generally rely on explicit group information, institution-level partitioning, or other prespecified heterogeneous structure.

\paragraph{Distributionally Robust Optimization}
Distributionally robust learning has emerged as a promising framework for handling various forms of distributional shift~\citep{lin2022distributionally, delage2010distributionally, mohajerin2018data, kuhn2019wasserstein, gao2024wasserstein, sagawa2019distributionally}. In particular, DRO has been widely studied for latent subpopulation shift, where robustness must be achieved without explicit group annotations. A prominent line of work improves worst-case performance by emphasizing high-loss regions of the distribution, as in $f$-DRO and related CVaR-type objectives~\citep{namkoong2016stochastic, duchi2021learning, hashimoto2018fairness, shapiro2021lectures}. Recent work has extended this idea to survival analysis without relying directly on predefined sensitive attributes~\citep{hu2022distributionally, hu2024fairness}. 

However, these methods can be highly vulnerable to outliers and may become unstable under contamination~\citep{zhai2021doro}. 
This issue is particularly important in survival analysis, where outliers are pervasive and can substantially affect model fitting and prediction. 
Existing robust survival methods mitigate outlier effects~\citep{heritier2009robust, sinha2019robust, farcomeni2012overview, luo2022penalized, farcomeni2011robust}, but do not directly address latent subpopulation shift. 
\citet{zhai2021doro} improve outlier robustness in $f$-DRO but focus on \lossrevision{additive losses} and do not cover non-decomposable survival losses with cross-sample dependence.
Thus, jointly addressing outlier contamination and latent subpopulation shift remains an important problem in survival analysis.
Independently of our work, \citet{huang2026nested} propose a nested DRO framework for robust learning with additive losses.
Their formulation does not directly apply to non-decomposable survival losses with cross-sample dependence.

\paragraph{Other Approaches to Subpopulation Shift}
Subpopulation shift has also been studied through latent subgroup discovery and annotation-free group robustness~\citep{liu2021just, hong2023harmonizing, stromberg2024robust}, group-aware priors~\citep{rudner2024mind}, outlier-robust group inference~\citep{zeng2022outlier}, and DRO frameworks with hierarchical or structured ambiguity sets~\citep{jo2025mitigating, kim2026distributionally}.
Our work focuses on jointly addressing subpopulation shift and outlier contamination in survival models with \lossrevision{non-decomposable losses}.

\section{Problem Setup}
In this section, we introduce the notation and formalize the problem setup. Although our primary interest lies in survival analysis, we first present our framework in a general statistical learning setting, and then discuss its application to survival models in Section~\ref{ssec:spec_to_survival}.

Let $Z \in \mathcal{Z}$ denote a random variable. For any distribution $Q$, we denote expectation under $Q$ by $\mathbb{E}_Q[\cdot]$.
For a positive integer $n$, let $[n] := \{1,\ldots,n\}$ denote the index set, and let $\mathcal{S}_{n-1} := \{ p \in \mathbb{R}^n : p_i \ge 0,\ \sum_{i=1}^n p_i = 1 \}$ denote the $(n-1)$-simplex.
Given a model $f_\theta$ parameterized by $\theta$, we denote the corresponding loss function by $\ell(Z, \theta)$.

We consider a learning problem under subpopulation shift with outlier contamination.
Specifically, the target distribution $P_{\rm test}$ is assumed to be a mixture of $K$ latent groups, $P_{\rm test} = \sum_{k=1}^{K} \pi_k P_k$, where $\sum_{k=1}^{K} \pi_k = 1$ and $\pi_k \in (0,1)$.
Here, $P_k$ denotes the distribution of the $k$-th latent group, and $\pi_k$ denotes its mixture proportion.
Let $P_{\mathrm{train}}$ denote the clean training distribution before outlier contamination, which may have different mixing proportions from $P_{\rm test}$.
The observed training distribution follows Huber's $\epsilon$-contamination model~\citep{huber1992robust}, $\widetilde P_{\mathrm{train}} = (1-\epsilon) P_{\mathrm{train}} + \epsilon \widetilde P$, where $\widetilde P$ is an arbitrary outlier distribution and $\epsilon \in [0,1)$ denotes the contamination level.
Let $\mathcal{D} = \{z_i\}_{i=1}^n$ be a dataset drawn i.i.d. from $\widetilde P_{\mathrm{train}}$.

Our goal is to learn a model that is robust to both outlier contamination in the training data and subpopulation shift under the latent mixture distribution $P_{\rm test}$.
A classical baseline is empirical risk minimization (ERM), which minimizes the average loss over the observed training data.
However, ERM does not account for distributional shift between the clean training distribution $P_{\mathrm{train}}$ and the target distribution $P_{\rm test}$, or for outlier contamination in the observed training distribution $\widetilde P_{\mathrm{train}}$.

A promising alternative is distributionally robust learning, which minimizes the worst-case risk over a set of plausible distributions.
Specifically, it solves
$$
\min_{\theta } \max_{Q \in \mathcal{B}_\rho(\bar{P})} \; \mathbb{E}_{Q}\big[\ell(Z, \theta)\big],
$$
where $\bar{P}$ is a nominal distribution and $\mathcal{B}_\rho(\bar{P})$ is an ambiguity set that captures distributional uncertainty around $\bar{P}$.
Two prominent DRO formulations for subpopulation shift are Group DRO \citep{sagawa2019distributionally} and $f$-DRO \citep{namkoong2016stochastic, duchi2021learning}. Group DRO directly optimizes the worst-case group risk but requires access to group labels, which are unavailable in many practical settings, including ours.
In contrast, $f$-DRO defines the ambiguity set via an $f$-divergence and achieves robustness by reweighting high-loss samples, which can effectively address latent subpopulation shift. However, it remains inherently vulnerable to outliers~\citep{zhai2021doro}. Therefore, it may fail to provide robustness when both subpopulation shift and outlier contamination are present.

\section{Proposed Method}
\lossrevision{We first introduce the general coupled DRO formulation and illustrate its structure with additive losses. We then extend the framework to non-decomposable survival losses, the main focus of this work.}

\subsection{Formulation} \label{ssec:formulation}
To motivate the proposed formulation, we first define a mixture ambiguity set. For a given nominal distribution $\bar P$ and a constant $\gamma \in (0,1]$, define
\begin{equation}
\mathcal{Q}_\gamma (\bar P)
=
\left\{
Q : \bar P = \gamma Q + (1-\gamma)\widetilde Q 
\; \textrm{for some} \;
\widetilde Q \in \mathcal{P}(\mathcal{Z})
\right\},
\end{equation}
where $\mathcal{P}(\mathcal{Z})$ denotes the set of all probability distributions on $\mathcal{Z}$. 
This set consists of $\gamma$-mixture components of $\bar P$, each of which is absolutely continuous with respect to $\bar P$, denoted by $Q \ll \bar P$.

This ambiguity set serves different roles in minimization and maximization. For a fixed $\theta$, consider
\begin{equation}
\min_{Q \in \mathcal{Q}_{1-\epsilon}(\bar P)} \mathbb{E}_Q[\ell(Z,\theta)]
\quad \text{and} \quad
\max_{Q \in \mathcal{Q}_\alpha(\bar P)} \mathbb{E}_Q[\ell(Z,\theta)].
\end{equation}
Here, $\epsilon \in [0,1)$ and $\alpha \in (0,1]$.
The minimization selects a distribution $Q \in \mathcal{Q}_{1-\epsilon}(\bar P)$ that minimizes the expected loss by assigning smaller weights to high-loss regions, including potential outliers. The resulting distribution can be viewed as a refined version of $\bar P$, with $\epsilon$ reflecting the degree of potential outlier contamination.
In contrast, the maximization selects a distribution $Q \in \mathcal{Q}_\alpha(\bar P)$ that maximizes the expected loss by assigning larger weights to high-loss regions. This distribution emphasizes regions that are particularly difficult under the current model, with $\alpha$ controlling the size of the subpopulation considered. This maximization is closely related to CVaR-type DRO~\citep{shapiro2021lectures}, which focuses on the high-loss tail of the distribution.

Motivated by these two roles, we propose the following coupled DRO problem:
\begin{equation}
\min_{\theta}
\left[
\min_{Q' \in \mathcal{Q}_{1-\epsilon}(\bar P)}
\max_{Q \in \mathcal{Q}_\alpha(Q')}
\mathbb{E}_{Q} \bigl[\ell(Z,\theta)\bigr]
\right].
\label{eq:nested_dro}
\end{equation}
We refer to the minimization over $Q'$ and the maximization over $Q$ as the outer minimization and the inner maximization, respectively. The formulation couples nominal refinement with adversarial reweighting through the dependence of the inner ambiguity set $\mathcal Q_\alpha(Q')$ on the refined distribution $Q'$.

\subsection{Structural analysis under additive losses}\label{ssec:additive_loss}
To facilitate understanding of the proposed formulation, we consider a class of losses for which the optimization structure becomes explicit. Specifically, given a weight vector 
$q \in \mathcal{S}_{n-1}$ over the dataset $\mathcal{D}$, we view the empirical risk as a function $\mathcal{L} : \mathcal{S}_{n-1} \times \Theta \to \mathbb{R}$. We call this class of losses additive if the empirical risk can be decomposed as $\mathcal{L}(q, \theta) = \sum_{i=1}^n q_i \ell(z_i,\theta)$, \lossrevision{where each loss term depends on a single sample.}
This class includes standard supervised learning losses such as the cross-entropy loss and the squared loss.
We now take the nominal distribution $\bar P$ to be supported on $\mathcal{D}$ with probability mass vector $\bar p \in \mathcal{S}_{n-1}$. Under this choice, each feasible distribution can be represented by its probability mass vector over $\mathcal{D}$. Thus, for additive losses, the proposed formulation in \eqref{eq:nested_dro} admits the following equivalent discrete representation:
\begin{equation}
\min_{\theta}
\left[
\min_{q' \in \mathcal{Q}_{1-\epsilon}(\bar p)}
\max_{q \in \mathcal{Q}_{\alpha}(q')}
\mathcal{L}(q,\theta)
\right].
\label{eq:nested_dro_discrete}
\end{equation}
Here, for a reference weight vector $p \in \mathcal{S}_{n-1}$ and $\gamma \in (0,1]$, the ambiguity set $\mathcal Q_\gamma(p)$ is defined as
\begin{equation}
\mathcal{Q}_{\gamma}(p)
=
\left\{
q \in \mathcal{S}_{n-1}
:
0 \le q_i \le \frac{p_i}{\gamma},\; \forall i \in [n]
\right\}.
\label{eq:discrete_ambiguity_set_general}
\end{equation}
Details of this finite-support representation of the ambiguity set are provided in Lemma~\ref{lem:Q1_equiv}.

The discrete coupled formulation in~\eqref{eq:nested_dro_discrete} admits an explicit two-stage characterization. Specifically, the outer minimization performs $\epsilon$-trimming, and the inner maximization reweights the highest-loss $\alpha$-fraction of the remaining probability mass. A precise characterization and the corresponding procedure are given in Appendix~\ref{appendix:additive_loss} and Algorithm~\ref{algo:additive_loss}. \lossrevision{Thus, our formulation includes CVaR-DORO~\citep{zhai2021doro} as a special case for additive losses.}

However, this structural result relies critically on the additivity of the loss and does not extend to settings where the loss depends on multiple samples simultaneously. In this paper, we extend the proposed framework beyond the additive setting to handle such non-decomposable losses. Survival analysis provides a primary example of this class, which we study in the following section.

\subsection{Non-decomposable survival losses} \label{ssec:spec_to_survival}
In this section, we focus on right-censored survival data and specify the corresponding survival loss, although the framework can be adapted to other censoring or truncation mechanisms. Let $Z = (X, Y, \Delta)$ be a random variable, where $X \in \mathcal{X}$ denotes the covariate, $T$ is the survival time, and $C$ is the censoring time. The observed time is given by $Y = \min\{T, C\}$, and the event indicator is $\Delta = \mathbb{I}(T \le C)$. Each observed sample is denoted by $z_i = (x_i, y_i, \delta_i)$.

\lossrevision{Survival losses} such as the Cox negative partial log-likelihood and ranking-based losses involve cross-sample dependence.
To apply the proposed coupled DRO formulation, we express the empirical survival risk as $\mathcal{L}(q,\theta)$ to make its dependence on the sample weights explicit. 
As a representative example, we consider the Cox proportional hazards model~\citep{cox1972regression}. 
Given $q \in \mathcal{S}_{n-1}$, the empirical risk is defined as the weighted negative partial log-likelihood:
\begin{equation*}
\mathcal{L}(q,\theta) = -\sum_{i=1}^n q_i \delta_i \left[ f_{\theta}(x_i) - \log \left( \sum_{j \in [n]: y_j \ge y_i} q_j \exp(f_{\theta}(x_j)) \right) \right].
\end{equation*}
\lossrevision{We incorporate this loss into the coupled formulation in \eqref{eq:nested_dro_discrete}. Because the sample weights also appear inside the logarithmic risk-set terms, $\mathcal{L}(q,\theta)$ is nonlinear in $q$. Thus, the closed-form trimming and reweighting structure derived for additive losses does not carry over directly.}
\lossrevision{Moreover, standard DRO formulations based on dual reformulations are difficult to apply directly to non-decomposable survival losses, since they typically rely on additive losses}~\citep{hu2024fairness, jin2025distributionally}.
We therefore develop a gradient-based optimization scheme that preserves the original survival loss structure with cross-sample dependence.

\vspace{-0.1em}
\section{Algorithm}
\vspace{-0.1em}
In this section, we present the learning algorithm and its intuition.
Throughout this section, we assume that $\mathcal{L}(q,\theta)$ is continuously differentiable in both $q$ and $\theta$, which enables gradient-based updates.
Our approach treats the optimization as a three-block problem over $(q, q', \theta)$ and adopts an alternating update strategy. 
A key challenge lies in updating $q'$, because the outer objective is defined implicitly through the inner maximization $\Phi_\alpha(q';\theta) := \max_{q \in \mathcal{Q}_\alpha(q')} \mathcal{L}(q,\theta)$.
Thus, its gradient with respect to $q'$ is not directly available. 
The following proposition characterizes a supergradient through the KKT multipliers of the inner problem.

\bigskip
\begin{proposition}[Outer supergradient via KKT multipliers]
\label{prop:outer_supergrad_main}
Fix $\theta$ and let $q^\star$ be a global maximizer of the inner problem.
Under the regularity and local stability assumptions stated in Appendix~\ref{appendix:non-decomposable_loss}, any KKT multiplier vector $\mu^\star$ for $q_i\le q'_i/\alpha$ at $q^\star$ yields a supergradient $\mu^\star/\alpha$ of $\Phi_\alpha(\cdot;\theta)$ at $q'$:
\[
\Phi_\alpha(q'+h;\theta)
\le \Phi_\alpha(q';\theta)
+\left\langle\frac{\mu^\star}{\alpha},h\right\rangle
+o(\|h\|).
\]
Here, $q'+h\in\mathcal S_{n-1}$ and $o(\|h\|)/\|h\|\to0$ as $h\to0$.
\end{proposition}

Based on Proposition~\ref{prop:outer_supergrad_main}, we can derive a supergradient with respect to $q'$ from the KKT multipliers $\mu^\star$ of the inner problem.
The detailed characterization of these multipliers and the derivation of this result are provided in Appendix~\ref{appendix:non-decomposable_loss}.

Combining the $q'$ update based on this supergradient with gradient steps for $q$ and $\theta$ yields the alternating procedure in Algorithm~\ref{alg:alt_nested_dro}.
We detail the full-batch updates below for clarity; stochastic mini-batch implementations are also possible~\citep{kvamme2019time, zeng2026mini}.

\begin{algorithm}[t]
\small
\caption{Learning algorithm for the coupled DRO objective}
\label{alg:alt_nested_dro}
\begin{algorithmic}[1]
\State \textbf{Input:} dataset $\mathcal D$; step sizes $\eta_q, \eta_{q'}, \eta_\theta$; update steps $T$, $K$; initial values $\theta^{(0)}, q^{(0)}, q'^{(0)}$

\For{$k = 0,1,2,\ldots, K-1$ }
    \State Set $q^{(k,0)} \leftarrow q^{(k)}$, $q'^{(k,0)} \leftarrow q'^{(k)}$
    
    \Comment{Update $(q, q')$}
    \For{$t = 0,\dots,T-1$}
        \State Update $q^{(k,t+1)}$ using the exponentiated gradient ascent and projection in \eqref{eq:q_update}
        \State Update $q'^{(k,t+1)}$ using the exponentiated gradient descent and projection in \eqref{eq:q'_update}
        
    \EndFor
    
    \State Set $q^{(k+1)} \leftarrow q^{(k,T)}$, $q'^{(k+1)} \leftarrow q'^{(k,T)}$
    
    \Comment{Update $\theta$}
    \State Update $\theta^{(k+1)}$ using the gradient descent step in \eqref{eq:theta_update}
    
\EndFor

\end{algorithmic}
\end{algorithm}

\paragraph{Update of $q$.}
Given $\theta$ and $q'$, we update $q$ to $\hat q$ using the projected exponentiated gradient ascent update~\citep{sagawa2019distributionally}
\begin{equation}
\begin{aligned}
\tilde{q}_i 
&\leftarrow 
\frac{q_i \exp \left( \eta_q [\nabla_q \mathcal{L}(q, \theta)]_i \right)}
{\sum_{j=1}^n q_j \exp \left( \eta_q [\nabla_q \mathcal{L}(q, \theta)]_j \right)}, 
\quad i = 1, \dots, n, \\
\hat{q} 
&\leftarrow 
\Pi_{\mathcal{Q}_{\alpha}(q')} (\tilde{q}),
\end{aligned}
\label{eq:q_update}
\end{equation}
where $\nabla_q \mathcal{L}(q,\theta)$ denotes the gradient of the map $q \mapsto \mathcal{L}(q,\theta)$, $\eta_q > 0$ is a step size, $\Pi_{\mathcal A}(\cdot)$ denotes projection onto the set $\mathcal A$, and $[x]_j$ denotes the $j$-th component of a vector $x$.

\paragraph{Update of $q'$.}
Given $\theta$ and $\hat q$, we update $q'$ to $\hat q'$ using projected exponentiated descent with a supergradient estimate $s$:
\begin{equation}
\begin{aligned}
\tilde q'_i 
&\leftarrow 
\frac{q'_i \exp \left( -\eta_{q'} s_i \right)}
{\sum_{j=1}^n q'_j \exp \left( -\eta_{q'} s_j \right)}, 
\quad i = 1, \dots, n, \\
\hat q' 
&\leftarrow 
\Pi_{\mathcal{Q}_{1-\epsilon}(\bar p)} (\tilde q'),
\end{aligned}
\label{eq:q'_update}
\end{equation}
where $\eta_{q'} > 0$ is a step size. For a nonempty interior set $\mathcal I$, we compute $s$ from $\hat q$ as follows:
$$
\mathcal{I}
:=
\left\{
i \in [n] : 0 < \hat q_i < \tfrac{1}{\alpha} q'_i
\right\},
\quad
\nu := \frac{1}{|\mathcal{I}|} \sum_{i\in\mathcal{I}} [\nabla_{q} \mathcal{L}(\hat q,\theta)]_i,
$$
$$
s_i =
\begin{cases}
\frac{1}{\alpha}\big([\nabla_{q} \mathcal{L}(\hat q,\theta)]_i - \nu \big)_+,
& \text{if } 0< \hat q_i = \tfrac{1}{\alpha} q'_i,\\[4pt]
0,
& \text{otherwise}.
\end{cases}
$$
If $\mathcal I=\emptyset$, we leave $q'$ unchanged, as in Algorithm~\ref{alg:alternating_nested_dro}.
\begin{figure}[!t]
\centering
\includegraphics[width=0.329\linewidth]{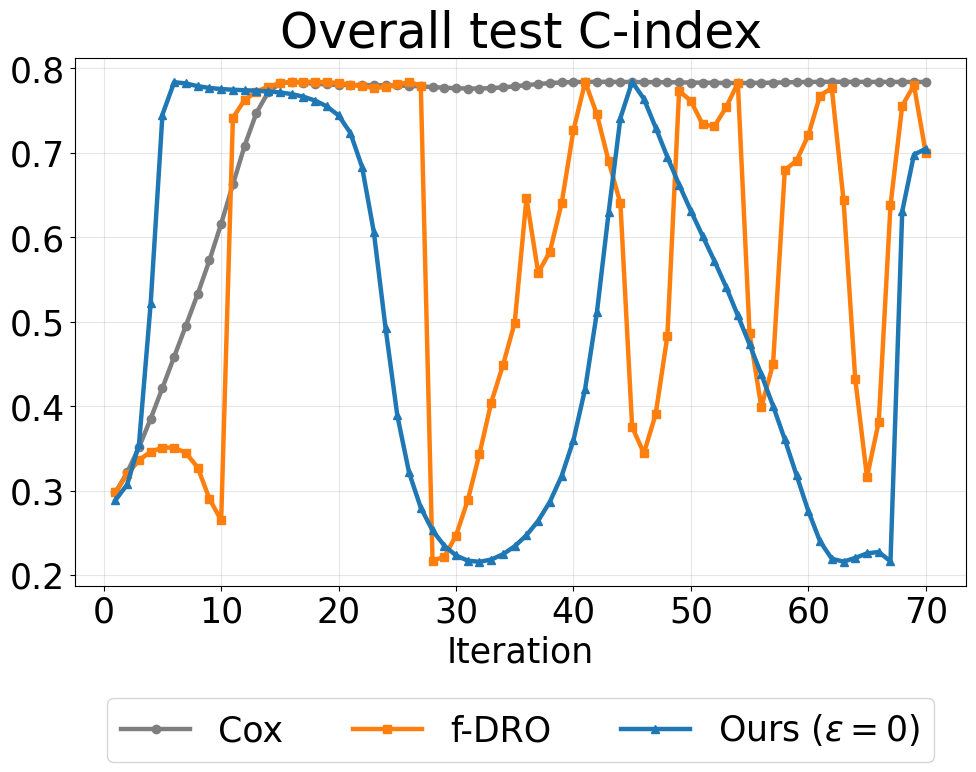}
\includegraphics[width=0.329\linewidth]{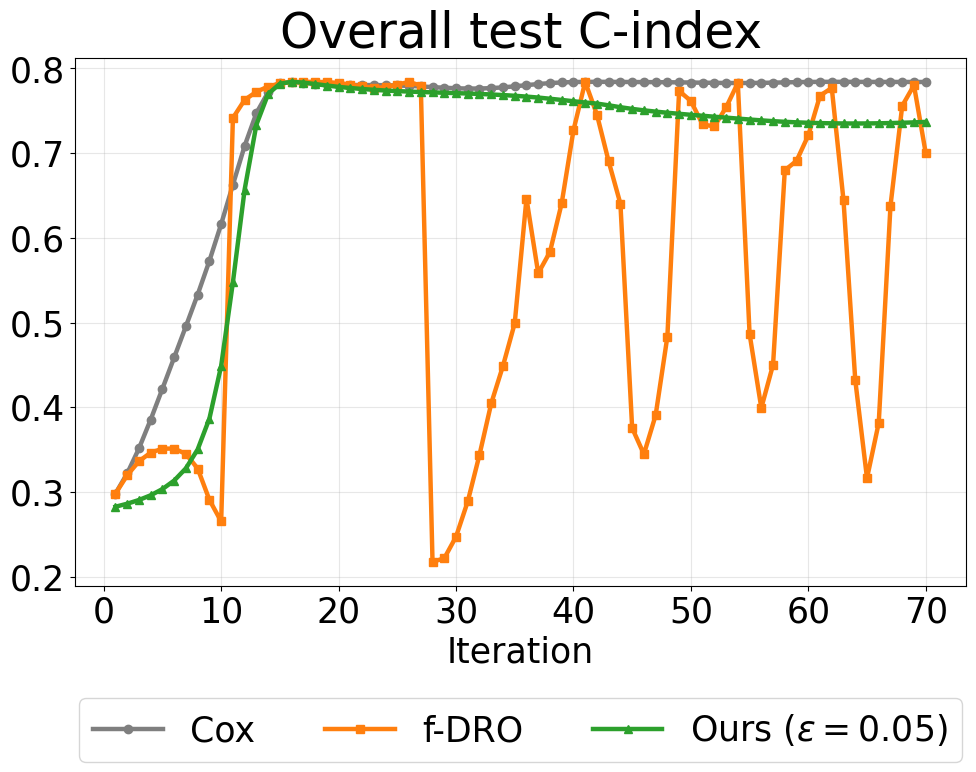}
\includegraphics[width=0.329\linewidth]{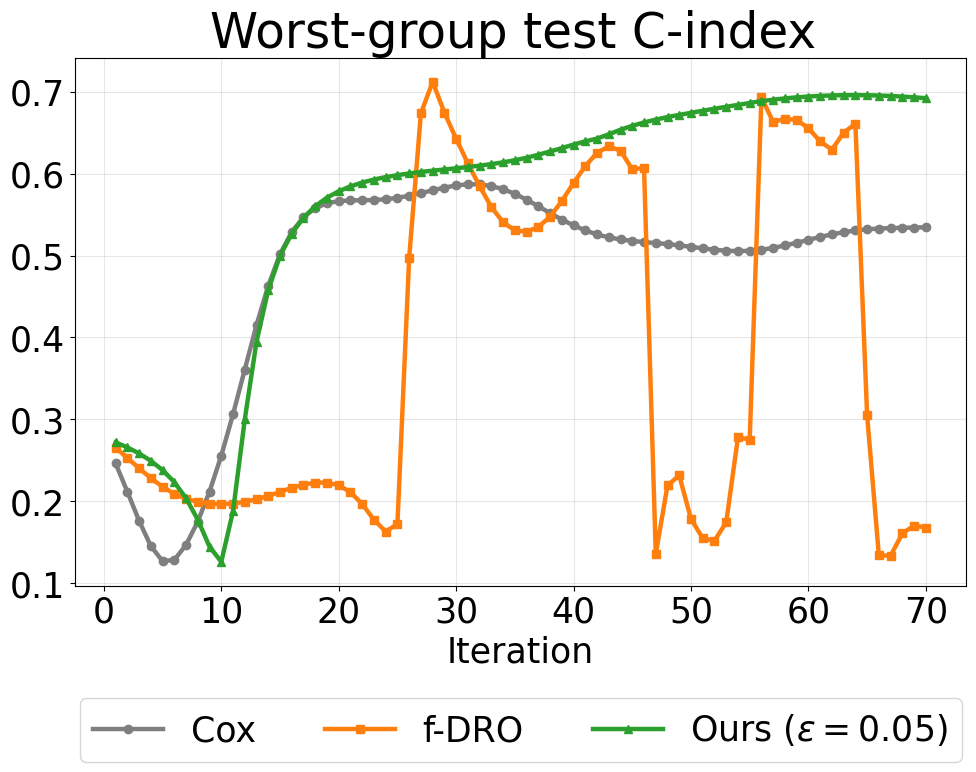}
\caption{
Test C-index trajectories under subpopulation shift with outlier contamination.
(Left) Overall test C-index for the Cox model, $f$-DRO, and the proposed method with $(\alpha, \epsilon) = (0.1, 0)$.
(Middle) Overall test C-index for the proposed method with $(\alpha, \epsilon) = (0.1, 0.05)$.
(Right) Worst-group test C-index for the proposed method with $(\alpha, \epsilon) = (0.1, 0.05)$.
For the proposed method, only $\epsilon$ is varied from $0$ to $0.05$, while all other training configurations are kept fixed.
}
\label{fig:subpopulation_outlier}
\end{figure}

\paragraph{Update of $\theta$.}
After alternating the updates of $q$ and $q'$ for several iterations, we update $\theta$ to $\hat \theta$ using the resulting $\hat q$ via a gradient descent step.
\begin{equation}
\hat \theta \leftarrow
\theta - \eta_\theta \nabla_\theta \mathcal{L}(\hat q,\theta),
\label{eq:theta_update}
\end{equation}
where $\eta_\theta > 0$ is a step size.

To interpret the overall procedure, recall the finite-support representation of our ambiguity set:
\begin{align*}
\mathcal{Q}_{1-\epsilon}(\bar p)
&=
\left\{
q' \in \mathcal{S}_{n-1}
:
0 \le q'_i \le \frac{\bar p_i}{1-\epsilon},\; \forall i\in[n]
\right\}, \\ 
\mathcal{Q}_{\alpha}(q')
&=
\left\{
q \in \mathcal{S}_{n-1}
:
0 \le q_i \le \frac{q'_i}{\alpha},\; \forall i\in[n]
\right\}.
\end{align*}
\lossrevision{The bound $q_i'/\alpha$ determines the maximum weight that can be assigned to sample $i$. Within these bounds, the inner update uses the gradient of the full survival loss to increase the loss. The outer update uses KKT multipliers from the inner problem to update $q'$, thereby adjusting the upper bounds $q_i'/\alpha$ on the sample weights. Both updates account for how changes in sample weights affect the full loss, including other samples' risk-set terms in the Cox model. Thus, $q$ and $q'$ can be updated while preserving the cross-sample interactions in the survival loss.}

Computational complexity is analyzed in Appendix~\ref{appendix:computational_complexity}.
A local convergence analysis is provided in Appendix~\ref{appendix:convergence}.

\vspace{-0.3em}
\section{Experiments}
\vspace{-0.3em}
\subsection{Simulation} \label{ssec:simulation}
We evaluate robustness under subpopulation shift with outlier contamination; results without outliers are provided in Appendix~\ref{appendix:sim_subpopulation_shift}.
For $f$-DRO, we use the $\chi^2$-divergence and sample splitting following~\citet{hu2024fairness}.

The training distribution mixes two groups with an outlier component:
$\widetilde P_{\rm train} = \pi^{\mathrm{tr}}_1 P_1 + \pi^{\mathrm{tr}}_2 P_2 + \tilde{\pi}^{\mathrm{tr}} \widetilde{P}$, 
with $(\pi^{\mathrm{tr}}_1, \pi^{\mathrm{tr}}_2, \tilde{\pi}^{\mathrm{tr}}) = (0.9, 0.08, 0.02)$ and $N_{\text{train}} = 3{,}000$. 
The test distribution is clean and is given by 
$P_{\rm test} = \pi^{\mathrm{te}}_1 P_1 + \pi^{\mathrm{te}}_2 P_2$, 
with $(\pi^{\mathrm{te}}_1, \pi^{\mathrm{te}}_2) = (0.8, 0.2)$ and $N_{\text{test}} = 5{,}000$. 
For each group $g \in \{1,2\}$, covariates are generated as $X \sim \mathcal{N}(0,2I_2)$, and survival times follow an exponential hazard model 
$h_g(t \mid x)=h_0\exp(x^\top\beta_g)$.
Censoring times are drawn from an exponential distribution, yielding 10\% censoring.
We set $\beta_1 = (1,1)$ and $\beta_2 = (1,-1)$ for the two groups, and use $\tilde{\beta} = (-2,-2)$ for the outlier component.

Figure~\ref{fig:subpopulation_outlier} compares overall and worst-group test C-index trajectories, where the latter is the minimum across the two groups~\citep{sagawa2019distributionally}.
Both $f$-DRO and the proposed method with $(\alpha,\epsilon)=(0.1,0)$ exhibit large fluctuations, consistent with their emphasis on high-loss observations under contamination.
The standard Cox model remains relatively stable but has a worst-group C-index of approximately $0.55$.
Setting $\epsilon=0.05$ stabilizes the proposed method's overall and worst-group trajectories and yields a worst-group C-index near $0.70$, illustrating the benefit of controlling outlier influence while protecting minority-group performance.

\lossrevision{Further simulation analyses, including nonlinear risk functions, higher-dimensional covariates, smaller training samples, heavier censoring, and IBS trajectories, are provided in Appendices~\ref{appendix:nonlinear_simulation}--\ref{appendix:ibs_trajectories}.}
Sensitivity to the robustness parameters $\alpha$ and $\epsilon$ is examined in Appendix~\ref{appendix:robustness_parameters}.

\begin{table*}[t]
\caption{Results on METABRIC under subpopulation shift with outlier contamination. The 10\% and 5\% settings specify the validation tail loss for Robust-Cox and the DRO-based methods; Elastic-Cox and Cox-NN use validation log-likelihood. Best results are in bold, and second-best results are underlined. Values are mean (standard deviation) over 10 independent runs.}
\label{tab:outlier_shift_metabric}
\centering
\fontsize{6.5}{7}\selectfont
\setlength{\tabcolsep}{5.25pt}
\renewcommand{\arraystretch}{0.95}
\begin{tabular}{llccc|cc}
\toprule
\toprule
\textbf{Model} & \textbf{Method}
& \textbf{IBS} $\downarrow$
& \textbf{Test log-likelihood} $\uparrow$
& \textbf{Overall C-index} $\uparrow$
& \textbf{Worst-group C-index} $\uparrow$
& \textbf{Test tail loss} $\downarrow$ \\
\midrule
\multirow{5}{*}{Linear (10\%)}
& Elastic-Cox
& \underline{0.176} (0.006)
& \underline{-5.061} (0.035)
& 0.618 (0.022)
& 0.501 (0.058)
& 6.305 (0.111) \\
& Robust-Cox
& \textbf{0.167} (0.005)
& \textbf{-5.016} (0.028)
& 0.665 (0.017)
& 0.562 (0.042)
& 6.229 (0.048) \\
& $f$-DRO
& 0.183 (0.005)
& -5.121 (0.033)
& \underline{0.667} (0.015)
& \underline{0.574} (0.030)
& \textbf{5.825} (0.015) \\
& Exact-Cox
& 0.184 (0.006)
& -5.121 (0.035)
& 0.583 (0.028)
& 0.459 (0.057)
& 6.179 (0.164) \\
& Ours
& 0.184 (0.005)
& -5.125 (0.034)
& \textbf{0.684} (0.015)
& \textbf{0.613} (0.028)
& \underline{5.835} (0.011) \\
\midrule
\multirow{5}{*}{Linear (5\%)}
& Elastic-Cox
& \underline{0.176} (0.006)
& \underline{-5.061} (0.035)
& 0.618 (0.022)
& 0.501 (0.058)
& 6.305 (0.111) \\
& Robust-Cox
& \textbf{0.168} (0.006)
& \textbf{-5.020} (0.026)
& \underline{0.658} (0.022)
& 0.547 (0.049)
& 6.240 (0.056) \\
& $f$-DRO
& 0.186 (0.003)
& -5.141 (0.018)
& 0.652 (0.033)
& \underline{0.558} (0.046)
& \textbf{5.824} (0.014) \\
& Exact-Cox
& 0.184 (0.006)
& -5.121 (0.035)
& 0.583 (0.028)
& 0.459 (0.057)
& 6.179 (0.164) \\
& Ours
& 0.187 (0.004)
& -5.145 (0.026)
& \textbf{0.680} (0.015)
& \textbf{0.605} (0.030)
& \underline{5.831} (0.007) \\
\midrule
\multirow{5}{*}{Nonlinear (10\%)}
& Cox-NN
& 0.181 (0.006)
& -5.101 (0.050)
& 0.601 (0.030)
& 0.508 (0.091)
& 6.561 (0.310) \\
& Robust-Cox
& \textbf{0.175} (0.005)
& \textbf{-5.063} (0.032)
& 0.635 (0.028)
& \underline{0.545} (0.049)
& 6.295 (0.161) \\
& $f$-DRO
& 0.182 (0.004)
& -5.119 (0.027)
& \underline{0.644} (0.036)
& 0.533 (0.047)
& \textbf{5.839} (0.041) \\
& Exact-Cox
& 0.197 (0.007)
& -5.210 (0.049)
& 0.456 (0.091)
& 0.395 (0.078)
& 6.226 (0.178) \\
& Ours
& \underline{0.176} (0.004)
& \underline{-5.074} (0.027)
& \textbf{0.661} (0.020)
& \textbf{0.585} (0.040)
& \underline{5.902} (0.014) \\
\midrule
\multirow{5}{*}{Nonlinear (5\%)}
& Cox-NN
& 0.181 (0.006)
& -5.101 (0.050)
& 0.601 (0.030)
& 0.508 (0.091)
& 6.561 (0.310) \\
& Robust-Cox
& \textbf{0.175} (0.005)
& \textbf{-5.063} (0.032)
& \underline{0.635} (0.028)
& \underline{0.545} (0.049)
& 6.295 (0.161) \\
& $f$-DRO
& 0.186 (0.004)
& -5.147 (0.029)
& 0.627 (0.036)
& 0.511 (0.051)
& \textbf{5.834} (0.021) \\
& Exact-Cox
& 0.197 (0.007)
& -5.210 (0.049)
& 0.456 (0.091)
& 0.395 (0.078)
& 6.226 (0.178) \\
& Ours
& \underline{0.177} (0.005)
& \underline{-5.079} (0.033)
& \textbf{0.656} (0.017)
& \textbf{0.570} (0.036)
& \underline{5.901} (0.026) \\

\bottomrule
\bottomrule
\end{tabular}
\end{table*}

\vspace{-0.2em}
\subsection{Real data analysis}
\vspace{-0.2em}
\begingroup
\makeatletter
\renewcommand{\paragraph}{\@startsection{paragraph}{4}{\z@}
  {1.2ex plus 0.4ex minus 0.2ex}{-1em}{\normalsize\bf}}
\makeatother
\paragraph{Baselines.}
We compare the proposed method with standard Cox models, $f$-DRO and its exact Cox variant (Exact-Cox) for latent subpopulation shift, and Robust-Cox~\citep{minder1996robust} for outlier contamination.
Both $f$-DRO and Exact-Cox use the $\chi^2$-divergence following~\citet{hu2024fairness}; $f$-DRO uses sample splitting, whereas Exact-Cox uses the full likelihood without sample splitting.
For nonlinear Robust-Cox, we extend the original estimator to a neural predictor.

\paragraph{Datasets.}
\lossrevision{We evaluate METABRIC~\citep{curtis2012genomic} and FLC~\citep{dispenzieri2012use}, defining four subgroups in each dataset based on clinical covariates following~\citet{hu2024fairness}. Group labels are unavailable during training and validation. Further details, including sample sizes, censoring rates, and subgroup definitions, are provided in Appendix~\ref{appendix:datasets}.}

\begin{table*}[!t]
\caption{Results on FLC under subpopulation shift with outlier contamination. Settings and notation follow Table~\ref{tab:outlier_shift_metabric}.}
\label{tab:outlier_shift}
\centering
\fontsize{6.5}{7}\selectfont
\setlength{\tabcolsep}{5.25pt}
\renewcommand{\arraystretch}{0.95}
\begin{tabular}{llccc|cc}
\toprule
\toprule
\textbf{Model} & \textbf{Method}
& \textbf{IBS} $\downarrow$
& \textbf{Test log-likelihood} $\uparrow$
& \textbf{Overall C-index} $\uparrow$
& \textbf{Worst-group C-index} $\uparrow$
& \textbf{Test tail loss} $\downarrow$ \\
\midrule

\multirow{5}{*}{Linear (10\%)}
& Elastic-Cox
& \underline{0.100} (0.003)
& \underline{-6.573} (0.066)
& 0.793 (0.003)
& 0.618 (0.038)
& 7.650 (0.114) \\
& Robust-Cox
& \textbf{0.096} (0.001)
& \textbf{-6.491} (0.016)
& 0.793 (0.005)
& 0.611 (0.040)
& 7.854 (0.031) \\
& $f$-DRO
& 0.107 (0.003)
& -6.708 (0.031)
& \underline{0.797} (0.004)
& \underline{0.626} (0.037)
& \underline{7.409} (0.035) \\
& Exact-Cox
& 0.103 (0.005)
& -6.634 (0.063)
& 0.789 (0.006)
& 0.609 (0.037)
& 7.581 (0.127) \\
& Ours
& 0.119 (0.002)
& -6.941 (0.009)
& \textbf{0.798} (0.004)
& \textbf{0.631} (0.036)
& \textbf{7.283} (0.006) \\
\midrule
\multirow{5}{*}{Linear (5\%)}
& Elastic-Cox
& \underline{0.100} (0.003)
& \underline{-6.573} (0.066)
& 0.793 (0.003)
& 0.618 (0.038)
& 7.650 (0.114) \\
& Robust-Cox
& \textbf{0.096} (0.001)
& \textbf{-6.491} (0.016)
& 0.793 (0.005)
& 0.611 (0.040)
& 7.854 (0.031) \\
& $f$-DRO
& 0.125 (0.002)
& -7.052 (0.008)
& \underline{0.795} (0.005)
& \underline{0.626} (0.030)
& \underline{7.294} (0.004) \\
& Exact-Cox
& 0.109 (0.010)
& -6.747 (0.192)
& 0.722 (0.140)
& 0.576 (0.085)
& 7.535 (0.072) \\
& Ours
& 0.119 (0.002)
& -6.941 (0.009)
& \textbf{0.798} (0.004)
& \textbf{0.631} (0.036)
& \textbf{7.283} (0.006) \\
\midrule
\multirow{5}{*}{Nonlinear (10\%)}
& Cox-NN
& \textbf{0.104} (0.006)
& \textbf{-6.617} (0.065)
& \underline{0.782} (0.010)
& 0.623 (0.029)
& 7.801 (0.208) \\
& Robust-Cox
& \underline{0.106} (0.007)
& \underline{-6.690} (0.126)
& 0.760 (0.031)
& 0.602 (0.030)
& 8.134 (0.324) \\
& $f$-DRO
& 0.111 (0.004)
& -6.725 (0.033)
& \textbf{0.785} (0.009)
& \underline{0.624} (0.028)
& \underline{7.447} (0.043) \\
& Exact-Cox
& 0.112 (0.008)
& -6.783 (0.155)
& 0.723 (0.125)
& 0.561 (0.081)
& 7.622 (0.188) \\
& Ours
& 0.116 (0.003)
& -6.837 (0.057)
& 0.777 (0.019)
& \textbf{0.630} (0.033)
& \textbf{7.371} (0.067) \\
\midrule
\multirow{5}{*}{Nonlinear (5\%)}
& Cox-NN
& \textbf{0.104} (0.006)
& \textbf{-6.617} (0.065)
& \textbf{0.782} (0.010)
& \underline{0.623} (0.029)
& 7.801 (0.208) \\
& Robust-Cox
& \underline{0.106} (0.007)
& \underline{-6.687} (0.129)
& 0.761 (0.032)
& 0.603 (0.031)
& 8.134 (0.324) \\
& $f$-DRO
& 0.126 (0.002)
& -7.055 (0.009)
& 0.779 (0.013)
& 0.620 (0.023)
& \underline{7.298} (0.004) \\
& Exact-Cox
& 0.119 (0.012)
& -6.948 (0.214)
& 0.625 (0.181)
& 0.508 (0.117)
& 7.579 (0.207) \\
& Ours
& 0.123 (0.004)
& -6.976 (0.043)
& \underline{0.781} (0.017)
& \textbf{0.628} (0.041)
& \textbf{7.284} (0.008) \\
\bottomrule
\bottomrule
\end{tabular}
\par\smallskip
\normalsize
\caption{Performance across hidden-layer sizes $h$ on METABRIC under outlier contamination. Settings and notation follow the 10\% validation setting in Table~\ref{tab:outlier_shift_metabric}, whose nonlinear block corresponds to $h=2$.}
\label{tab:network_width}
\centering
\fontsize{6.5}{7}\selectfont
\setlength{\tabcolsep}{5.25pt}
\renewcommand{\arraystretch}{0.95}
\begin{tabular}{llccc|cc}
\toprule
\toprule
\textbf{$h$} & \textbf{Method}
& \textbf{IBS} $\downarrow$
& \textbf{Test log-likelihood} $\uparrow$
& \textbf{Overall C-index} $\uparrow$
& \textbf{Worst-group C-index} $\uparrow$
& \textbf{Test tail loss} $\downarrow$ \\
\midrule
\multirow{5}{*}{4} & Cox-NN & \underline{0.183} (0.007) & \underline{-5.118} (0.047) & 0.597 (0.012) & 0.474 (0.073) & 6.658 (0.188) \\
 & Robust-Cox & \underline{0.183} (0.009) & -5.139 (0.062) & 0.612 (0.019) & 0.506 (0.065) & 6.914 (0.322) \\
 & $f$-DRO & 0.185 (0.004) & -5.135 (0.021) & \underline{0.627} (0.028) & \underline{0.519} (0.036) & \textbf{5.832} (0.021) \\
 & Exact-Cox & 0.199 (0.006) & -5.225 (0.036) & 0.452 (0.060) & 0.391 (0.054) & 6.245 (0.079) \\
 & Ours & \textbf{0.177} (0.004) & \textbf{-5.077} (0.023) & \textbf{0.646} (0.027) & \textbf{0.546} (0.042) & \underline{5.946} (0.045) \\
\midrule
\multirow{5}{*}{8} & Cox-NN & 0.189 (0.005) & -5.184 (0.068) & 0.581 (0.029) & 0.482 (0.043) & 6.600 (0.662) \\
 & Robust-Cox & 0.186 (0.007) & -5.179 (0.066) & \underline{0.611} (0.017) & \underline{0.504} (0.061) & 7.106 (0.295) \\
 & $f$-DRO & \underline{0.185} (0.005) & \underline{-5.138} (0.037) & \underline{0.611} (0.012) & 0.492 (0.055) & \textbf{5.870} (0.053) \\
 & Exact-Cox & 0.201 (0.011) & -5.243 (0.070) & 0.439 (0.072) & 0.387 (0.073) & 6.344 (0.178) \\
 & Ours & \textbf{0.175} (0.005) & \textbf{-5.068} (0.037) & \textbf{0.654} (0.027) & \textbf{0.582} (0.045) & \underline{5.971} (0.084) \\
\midrule
\multirow{5}{*}{16} & Cox-NN & 0.191 (0.006) & -5.191 (0.070) & 0.571 (0.032) & 0.467 (0.074) & 6.313 (0.731) \\
 & Robust-Cox & 0.190 (0.009) & -5.231 (0.076) & \underline{0.608} (0.021) & 0.477 (0.051) & 7.320 (0.374) \\
 & $f$-DRO & \underline{0.186} (0.004) & \underline{-5.140} (0.028) & 0.600 (0.020) & \underline{0.494} (0.056) & \textbf{5.888} (0.049) \\
 & Exact-Cox & 0.201 (0.010) & -5.250 (0.062) & 0.465 (0.064) & 0.394 (0.045) & 6.367 (0.207) \\
 & Ours & \textbf{0.175} (0.003) & \textbf{-5.065} (0.028) & \textbf{0.653} (0.017) & \textbf{0.568} (0.057) & \underline{5.980} (0.089) \\
\bottomrule
\bottomrule
\end{tabular}
\end{table*}

\paragraph{Experimental Setup.}
Although the framework applies to general survival models, we consider linear Cox models and neural Cox models (Cox-NN) with a single hidden layer of $h=2$ units.
\lossrevision{Motivated by \citet{duchi2021learning}, we select hyperparameters for Robust-Cox and the DRO-based methods by minimizing validation tail loss without using group labels.}
We evaluate overall performance using the integrated Brier score (IBS), test log-likelihood, and C-index, and robustness using worst-group C-index and 5\% test tail loss.
Details of the datasets, data splits, baselines, training, and hyperparameter selection are provided in Appendices~\ref{appendix:datasets}--\ref{appendix:training_selection}.

\paragraph{Outlier Construction.}
We evaluate subpopulation shift with outlier contamination; results without outliers appear in Appendix~\ref{appendix:real_subpopulation_shift}.
We generate new survival times by inverting the fitted Cox risks of the highest-risk training samples: the top 2\% for linear models and 5\% for nonlinear models.
These settings are motivated by survival outliers whose event times are unusually long or short relative to predictions based on their covariates~\citep{luo2022penalized}.
We use risk inversion to introduce such departures in a small fraction of the training samples, with effects that may vary across datasets and model classes.
Construction details and additional contamination experiments are provided in Appendices~\ref{appendix:datasets}, \ref{appendix:additional_outliers}, and~\ref{appendix:contamination_sensitivity}.

\paragraph{Results under outlier contamination.}
\lossrevision{On METABRIC (Table~\ref{tab:outlier_shift_metabric}), the proposed method outperforms all baselines in both overall and worst-group C-index in every setting, while achieving the second-lowest test tail loss.
On FLC (Table~\ref{tab:outlier_shift}), the results show a trade-off between average predictive performance and robustness: the proposed method achieves the highest worst-group C-index and the lowest test tail loss in every setting, but has higher IBS and lower test log-likelihood than the standard Cox models and Robust-Cox.
Additionally, we conduct experiments with appended outliers on FLC using linear models and with 2\% and 5\% contamination across all dataset--model combinations (Appendices~\ref{appendix:additional_outliers} and~\ref{appendix:contamination_sensitivity}).
Across these experiments, the proposed method consistently achieves the lowest or second-lowest test tail loss.}
Runtime comparisons are provided in Appendix~\ref{appendix:runtime}.

\paragraph{Ablation across network configurations.}
We evaluate larger networks on METABRIC using hidden-layer sizes $h=4,8,16$ under the same contamination and model-selection settings.
As shown in Table~\ref{tab:network_width}, the proposed method achieves the highest overall and worst-group C-index, the lowest IBS, and the highest test log-likelihood in all three configurations.
Its test tail loss is consistently second lowest, with $f$-DRO attaining the lowest values.
These results show that the gains in overall and worst-group discrimination persist across the larger networks evaluated here.
\par
\endgroup

\vspace{-0.5em}
\section{Conclusion}
\vspace{-0.5em}
\lossrevision{We proposed a coupled DRO framework for survival analysis that jointly addresses subpopulation shift and outlier contamination. The formulation couples nominal refinement with adversarial reweighting while preserving cross-sample interactions in non-decomposable survival losses. We derived outer updates from the inner problem's KKT multipliers, yielding an alternating algorithm that directly optimizes the coupled objective. Experiments demonstrate improved worst-group C-index and consistently low test tail loss under contamination, with competitive overall C-index. Future work could develop ambiguity sets that also capture within-group covariate shifts while retaining outlier robustness.}

\clearpage
\subsection*{AI use statement}
We used generative AI tools to assist with manuscript drafting and editing and implementation details.
The authors reviewed all AI-assisted content and take responsibility for the final manuscript.


\bibliographystyle{iclr2027_conference}
\bibliography{ref}

\appendix
\clearpage
\raggedbottom
\renewcommand{\floatpagefraction}{0.8}
\section{Appendix}
\label{appendix:overview}

\noindent{\large\bfseries Contents of the appendix}\par\medskip
\makeatletter
\begingroup
\small
\hypersetup{hidelinks}
\setcounter{tocdepth}{3}
\@starttoc{apc}
\endgroup
\let\survival@addcontentsline\addcontentsline
\renewcommand{\addcontentsline}[3]{%
  \survival@addcontentsline{#1}{#2}{#3}%
  \def\survival@tocfile{#1}%
  \def\survival@maintoc{toc}%
  \ifx\survival@tocfile\survival@maintoc
    \survival@addcontentsline{apc}{#2}{#3}%
  \fi
}
\makeatother

\clearpage
\subsection{Formulation and algorithm details}\label{appendix:formulation_algorithm}

For additive losses, $\mathcal L(q,\theta)=\sum_{i=1}^n q_i\ell(z_i,\theta)$ is linear in $q$, so the optimal weights follow the ordering of the losses.
We first derive this trimming and reweighting structure, then use KKT multipliers to obtain gradient-based updates for non-decomposable losses such as the weighted Cox loss.
We also analyze computational complexity and local convergence.

For fixed $\theta$, recall the inner value function
\begin{align}
    \Phi_\alpha(q'; \theta)
    :=
    \max_{q \in \mathcal{Q}_\alpha(q')} \mathcal L(q,\theta).
    \label{eq:inner_maximization}
\end{align}
We omit $\theta$ when it is fixed.

\subsubsection{\lossrevision{Additive losses}}\label{appendix:additive_loss}
Write $\ell_i:=\ell(z_i,\theta)$ and $\Phi_\alpha(q'):=\Phi_\alpha(q';\theta)$.
We first express the ambiguity set as bounds on sample weights, then derive $\epsilon$-trimming for $\alpha=1$ and the two-stage solution for general $\alpha$.
The uniform-weight case gives the explicit allocation in Corollary~\ref{cor:exact_trimming_then_exact_alpha}.
\begin{lemma}[Equivalent discrete representation of the ambiguity set]
	\label{lem:Q1_equiv}
	Let $\bar P$ be the nominal distribution supported on $\mathcal D$, with corresponding weight vector $\bar p \in \mathcal{S}_{n-1}$. 
	For $\gamma \in (0,1]$, the set $\mathcal Q_{\gamma}(\bar P)$ admits the equivalent finite-support representation~\eqref{eq:discrete_ambiguity_set_general}, \ie,
	\begin{equation*}
	\mathcal{Q}_{\gamma}(\bar p)
	=
	\left\{
	q \in \mathcal{S}_{n-1}
	:
	0 \le q_i \le \frac{\bar p_i}{\gamma},\; \forall i \in [n]
	\right\}.
	\end{equation*}
	Moreover, every such $q$ can be reparameterized as
	\begin{align}
	q_i = \frac{\bar p_i}{\gamma} w_i,
	\qquad
	w \in [0,1]^n, 
	\qquad
	\sum_{i=1}^n \bar p_i w_i = \gamma.
	\label{eq:q_w_general}
	\end{align}
\end{lemma}

\begin{proof}
	We first consider the case $\gamma \in (0,1)$.

	($\Rightarrow$)
	Suppose that $Q \in \mathcal{Q}_{\gamma}(\bar P)$. 
	Since $\bar P$ is supported on the dataset $\mathcal D = \{z_i\}_{i=1}^n$, there exists a weight vector $\bar p \in \mathcal{S}_{n-1}$ such that
	\[
	\bar P = \sum_{i=1}^n \bar p_i \delta_{z_i},
	\]
	where $\delta_{z_i}$ denotes the Dirac measure at $z_i$. 
	By the definition of $\mathcal{Q}_{\gamma}(\bar P)$, there exists a probability distribution $\widetilde Q \in \mathcal{P}(\mathcal Z)$ such that
	\[
	\bar P = \gamma Q + (1-\gamma)\widetilde Q.
	\]
	Since $\bar P$ is supported on $\mathcal D$, the above identity implies that $Q$ and $\widetilde Q$ are also supported on $\mathcal D$. 
	Hence there exist weight vectors $q, \tilde q \in \mathcal{S}_{n-1}$ such that
	\[
	Q = \sum_{i=1}^n q_i \delta_{z_i},
	\qquad 
	\widetilde Q = \sum_{i=1}^n \tilde q_i \delta_{z_i}.
	\]
	By comparing the coefficients of the Dirac measures at each $z_i$, we obtain
	\[
	\bar p = \gamma q + (1-\gamma)\tilde q.
	\]
	Taking the $i$-th coordinate and using $\tilde q_i \ge 0$, we have
	\[
	\bar p_i = \gamma q_i + (1-\gamma)\tilde q_i \ge \gamma q_i,
	\]
	which implies
	\[
	q_i \le \frac{\bar p_i}{\gamma}.
	\]
	Since $q \in \mathcal{S}_{n-1}$, we also have $q_i \ge 0$ for all $i$ and $\sum_{i=1}^n q_i = 1$. 
	Hence $q \in \mathcal{Q}_{\gamma}(\bar p)$ in \eqref{eq:discrete_ambiguity_set_general}.

	($\Leftarrow$)
	Conversely, suppose that $q \in \mathcal{S}_{n-1}$ satisfies
	\[
	0 \le q_i \le \frac{\bar p_i}{\gamma},
	\qquad \forall i \in [n].
	\]
	Define a vector $\tilde q \in \mathbb{R}^n$ by
	\[
	\tilde q_i
	:=
	\frac{\bar p_i - \gamma q_i}{1-\gamma},
	\qquad i \in [n].
	\]
	Since $q_i \le \bar p_i/\gamma$, we have $\tilde q_i \ge 0$ for all $i \in [n]$. 
	Moreover,
	\[
	\sum_{i=1}^n \tilde q_i
	=
	\frac{1}{1-\gamma}
	\sum_{i=1}^n
	\bigl(\bar p_i - \gamma q_i\bigr)
	=
	\frac{1}{1-\gamma}
	\bigl(1-\gamma \cdot 1\bigr)
	=
	1.
	\]
	Hence $\tilde q \in \mathcal{S}_{n-1}$.

	Now define probability distributions
	\[
	Q = \sum_{i=1}^n q_i \delta_{z_i},
	\qquad
	\widetilde Q = \sum_{i=1}^n \tilde q_i \delta_{z_i}.
	\]
	Then
	\[
	\gamma Q + (1-\gamma)\widetilde Q
	=
	\sum_{i=1}^n \bigl(\gamma q_i + (1-\gamma)\tilde q_i\bigr)\delta_{z_i}
	=
	\sum_{i=1}^n \bar p_i \delta_{z_i}
	=
	\bar P.
	\]
	Therefore, $Q \in \mathcal{Q}_{\gamma}(\bar P)$.
    
	We now show the reparameterization in \eqref{eq:q_w_general}. 
	Suppose that $q$ satisfies \eqref{eq:discrete_ambiguity_set_general}. 
	For each $i \in [n]$ with $\bar p_i > 0$, define
	\[
	w_i := \frac{\gamma q_i}{\bar p_i}.
	\]
	If $\bar p_i = 0$, then \eqref{eq:discrete_ambiguity_set_general} implies $q_i = 0$, and we define $w_i := 0$.

	Then $w_i \ge 0$ for all $i$. 
	Moreover, from $q_i \le \bar p_i/\gamma$, it follows that $w_i \le 1$ whenever $\bar p_i > 0$. 
	Hence $w \in [0,1]^n$. 
	By construction,
	\[
	q_i = \frac{\bar p_i}{\gamma} w_i,
	\qquad \forall i \in [n].
	\]
	Finally,
	\[
	\sum_{i=1}^n \bar p_i w_i
	=
	\sum_{i:\bar p_i>0} \bar p_i \frac{\gamma q_i}{\bar p_i}
	=
	\gamma \sum_{i=1}^n q_i
	=
	\gamma,
	\]
	since $q \in \mathcal{S}_{n-1}$.

	If $\gamma=1$, then the condition
	\[
	\bar P = Q
	\]
	implies $q=\bar p$. 
	The discrete constraint becomes $0 \le q_i \le \bar p_i$ with $q \in \mathcal{S}_{n-1}$, which also implies $q=\bar p$. 
	The reparameterization holds with $w_i=1$ for $\bar p_i>0$ and $w_i=0$ for $\bar p_i=0$. 
	This completes the proof.
\end{proof}

\bigskip
Using the representation in Lemma~\ref{lem:Q1_equiv}, we first consider the special case $\alpha = 1$. 
In this case, the inner ambiguity set $\mathcal{Q}_\alpha(q')$ reduces to the singleton $\{q'\}$, and the coupled DRO problem reduces to
\begin{align}
	\min_{q' \in \mathcal{Q}_{1-\epsilon}(\bar p)}
	\sum_{i=1}^n q'_i \ell_i .
	\label{eq:alpha_1_problem}
\end{align}
The minimizer assigns mass to smaller-loss samples first, yielding $\epsilon$-trimming.
\bigskip
\begin{proposition}[Exact $\epsilon$-trimming under $\alpha=1$]
	\label{prop:epsilon_trimming_alpha1_general}
	Assume $\alpha = 1$. Consider the problem
	\begin{align*}
		q'^\star 
		\in 
		\argmin_{q' \in \mathcal{Q}_{1-\epsilon}(\bar p)}
		\sum_{i=1}^n q'_i \ell_i .
		\label{eq:prop1_problem}
	\end{align*}
	Then there exists an optimal solution of the form
	\begin{align*}
		q_i'^\star
		=
		\frac{\bar p_i w_i^\star}{1-\epsilon},
	\end{align*}
	where $w^\star \in [0,1]^n$ satisfies
	\begin{align*}
		\sum_{i=1}^n \bar p_i w_i^\star = 1-\epsilon,
	\end{align*}
	and, after reindexing so that
	\[
	\ell_1 \le \ell_2 \le \cdots \le \ell_n,
	\]
	there exists $k \in [n]$ such that
	\[
	w_i^\star =
	\begin{cases}
		1, & i < k, \\
		\gamma, & i = k, \\
		0, & i > k,
	\end{cases}
	\]
	for some $\gamma \in [0,1]$ chosen to satisfy the probability mass constraint $\sum_{i=1}^n q'^*_i = 1$.
\end{proposition}

\begin{proof}
	The argument transfers mass from larger to smaller losses without increasing the objective.
	By Lemma~\ref{lem:Q1_equiv}, any $q'\in\mathcal{Q}_{1-\epsilon}(\bar p)$ admits the representation
	\[
	q'_i=\frac{\bar p_i w_i}{1-\epsilon},
	\qquad 
	w\in[0,1]^n,
	\qquad
	\sum_{i=1}^n \bar p_i w_i = 1-\epsilon.
	\]
	Then, \eqref{eq:alpha_1_problem} reduces to
	\begin{align*}
		\min_{w\in[0,1]^n}
		\quad 
		& \sum_{i=1}^n \bar p_i w_i \ell_i
		\quad
		\text{s.t.}
		\quad
		\sum_{i=1}^n \bar p_i w_i = 1-\epsilon.
	\end{align*}
	Let $m_i := \bar p_i w_i$. The problem is equivalently
	\[
	\min_{0 \le m_i \le \bar p_i}
	\sum_{i=1}^n \ell_i m_i
	\quad
	\text{s.t.}
	\quad
	\sum_{i=1}^n m_i = 1-\epsilon.
	\]
	The objective is linear in the allocated mass $m_i$ with coefficient $\ell_i$,
	while $\bar p_i$ acts as an upper bound.
	Reindex so that $\ell_1 \le \ell_2 \le \cdots \le \ell_n$.
	Suppose there exist $i<j$ such that 
	$m_i < \bar p_i$ and $m_j > 0$.
	Then for any 
	$0 < \delta \le \min\{\bar p_i - m_i,\; m_j\}$,
	the perturbation
	$$
	m_i \leftarrow m_i + \delta,
	\qquad
	m_j \leftarrow m_j - \delta
	$$
	preserves feasibility and changes the objective by
	$
	\delta(\ell_i - \ell_j) \le 0
	$,
	with strict inequality if $\ell_i < \ell_j$. Hence an optimal solution assigns mass in increasing loss order up to the bounds, including when losses are tied.

	Thus there exists an index $k\in[n]$ such that
	\[
	m_i^\star =
	\begin{cases}
		\bar p_i, & i < k,\\
		\widetilde{m}_k, & i = k,\\
		0, & i > k,
	\end{cases}
	\]
	for some $\widetilde{m}_k \in [0,\bar p_k]$ chosen to satisfy the probability mass constraint.
	Equivalently, there exist $k\in[n]$ and $\gamma\in[0,1]$ such that
	\[
	w_i^\star =
	\begin{cases}
		1, & i < k,\\
		\gamma, & i = k,\\
		0, & i > k,
	\end{cases}
	\]
	with $\sum_{i=1}^n \bar p_i w_i^\star = 1-\epsilon$.
	This completes the proof.
\end{proof}

\bigskip
The next result extends this structure to the general case $\alpha \in (0,1]$ under $\epsilon \in [0, 1)$.  
We first show that the inner maximization problem $\Phi_\alpha(q')$ admits a weighted CVaR representation.
\begin{lemma}[Weighted CVaR dual representation of $\Phi_\alpha$]
	\label{lem:Phi_weighted_CVaR}
	For any $q' \in \mathcal{S}_{n-1}$, the inner maximization problem $\Phi_\alpha(q')$ admits the dual representation
	\[
	\Phi_\alpha(q')
	=
	\min_{\tau \in \mathbb{R}}
	\left\{
	\tau + \frac{1}{\alpha}
	\sum_{i=1}^n q_i' (\ell_i - \tau)_+
	\right\},
	\]
    where $(x)_+ := \max\{x,0\}$.
\end{lemma}

\begin{proof}
	Recall that the inner maximization problem is
	\[
	\Phi_\alpha(q')
	=
	\max
	\left\{
	\sum_{i=1}^n q_i \ell_i :
	q \in \mathcal{S}_{n-1},\ 
	0 \le q_i \le \frac{1}{\alpha} q_i',
	\ \forall i \in [n]
	\right\}.
	\]
	This is precisely the standard primal formulation of weighted CVaR under the reference weight vector $q'$. By the dual representation of CVaR~\citep{shapiro2021lectures}, we obtain
	\[
	\Phi_\alpha(q')
	=
	\min_{\tau \in \mathbb{R}}
	\left\{
	\tau + \frac{1}{\alpha}
	\sum_{i=1}^n q_i' (\ell_i - \tau)_+
	\right\}.
	\]
	This completes the proof.
\end{proof}

\bigskip

The weighted CVaR representation yields the following two-stage solution.

\bigskip
\begin{theorem}[Two-stage structure of the coupled DRO problem]
	\label{thm:epsilon_then_alpha_structural}
	
	There exists an optimal solution with the following structure. 
	After reordering the indices so that
	$
	\ell_1 \le \ell_2 \le \cdots \le \ell_n,
	$ i.e., in increasing order of losses,
	the solution exhibits a two-stage structure consisting of an outer and an inner component.
	
	\smallskip
	\noindent\textbf{(Outer structure)}
	The outer minimization assigns mass to smaller-loss samples first. Specifically,
	\[
	q_i'^\star=\frac{\bar p_i w_i^\star}{1-\epsilon},
	\qquad
	w^\star\in[0,1]^n,
	\qquad
	\sum_{i=1}^n \bar p_i w_i^\star=1-\epsilon,
	\]
	where
	\[
	w_i^\star=
	\begin{cases}
		1, & i<k,\\
		\gamma, & i=k,\\
		0, & i>k,
	\end{cases}
	\]
	for some $k\in[n]$ and $\gamma\in[0,1]$.
	
	\smallskip
	\noindent\textbf{(Inner structure)}
	Given $q'^\star$, the inner maximization concentrates mass on larger-loss samples. In particular,
	\[
	q_i^\star=
	\begin{cases}
		0, & i<t,\\
		\eta\,\frac{1}{\alpha}q_t'^\star, & i=t,\\
		\frac{1}{\alpha}q_i'^\star, & i>t,
	\end{cases}
	\qquad
	\sum_{i=1}^n q_i^\star=1,
	\]
	for some $t\in[n]$ and $\eta\in[0,1]$.
	
\end{theorem}

\smallskip
\begin{proof}
The CVaR representation turns the outer problem into a joint minimization over $q'$ and a threshold.
Mass transfers then characterize both the outer and inner optimizers.

By Lemma~\ref{lem:Phi_weighted_CVaR}, for any $q'\in\mathcal{S}_{n-1}$,
\[
\Phi_\alpha(q')
=
\min_{\tau\in\mathbb R}
\left\{
\tau+\frac{1}{\alpha}\sum_{i=1}^n q'_i(\ell_i-\tau)_+
\right\}.
\]

Therefore, the DRO objective can be written as the joint minimization
\begin{align}
\min_{q'\in\mathcal Q_{1-\epsilon}(\bar p)}\Phi_\alpha(q')
=
\min_{q'\in\mathcal Q_{1-\epsilon}(\bar p),\ \tau\in\mathbb R}
\left\{
\tau+\frac{1}{\alpha}\sum_{i=1}^n q'_i(\ell_i-\tau)_+
\right\}.
\label{eq:joint_problem}
\end{align}

Let $q'^\star$ be an outer minimizer of $\Phi_\alpha(q')$, and choose
\[
\tau^\star\in\argmin_{\tau \in \mathbb R}
\left\{
\tau+\frac{1}{\alpha}\sum_i q_i'^\star(\ell_i-\tau)_+
\right\}.
\]

For any $q' \in \mathcal Q_{1-\epsilon}(\bar p)$ and $\tau \in \mathbb R$,
\[
\tau+\frac{1}{\alpha}\sum_{i=1}^n q'_i(\ell_i-\tau)_+
\ge
\Phi_\alpha(q')
\ge
\Phi_\alpha(q'^\star)
=
\tau^\star+\frac{1}{\alpha}\sum_{i=1}^n q_i'^\star(\ell_i-\tau^\star)_+.
\]
Hence $(q'^\star,\tau^\star)$ attains the minimum in \eqref{eq:joint_problem}.

	\smallskip
	\noindent\textbf{(Outer structure).}
	Define $a_i^\star:=(\ell_i-\tau^\star)_+$.
	For fixed $\tau^\star$, the objective is linear in $q'$, so $q'^\star$ solves
	\[
	\min_{q'\in\mathcal Q_{1-\epsilon}(\bar p)}
	\sum_{i=1}^n q'_i a_i^\star.
	\]
	By Lemma~\ref{lem:Q1_equiv}, write $q'_i=\bar p_i w_i/(1-\epsilon)$ with
	$w\in[0,1]^n$ and $\sum_i \bar p_i w_i=1-\epsilon$.
	Equivalently, letting $m_i:=\bar p_i w_i$, we obtain
	\[
	\min_{0\le m_i\le \bar p_i}
	\sum_{i=1}^n a_i^\star m_i
	\quad
	\text{s.t.}
	\quad
	\sum_{i=1}^n m_i=1-\epsilon.
	\]
	Reindex the samples so that $\ell_1\le\cdots\le\ell_n$.
	Since $a_i^\star=(\ell_i-\tau^\star)_+$ is non-decreasing in $\ell_i$,
	this indexing also satisfies $a_1^\star\le\cdots\le a_n^\star$.
	If there exist $i<j$ with $m_i<\bar p_i$ and $m_j>0$, then for any $0<\delta\le\min\{\bar p_i-m_i,\ m_j\}$,
	the perturbation
	$m_i\leftarrow m_i+\delta$, $m_j\leftarrow m_j-\delta$
	preserves feasibility and changes the objective by $\delta(a_i^\star-a_j^\star)\le0$
	with strict decrease if $a_i^\star<a_j^\star$.
	Therefore, an optimal solution can be chosen that fills indices in this order, including when the $a_i^\star$ are tied.
	Consequently, there exist
	$k\in[n]$ and $\gamma\in[0,1]$ such that
	\[
	w_i^\star=
	\begin{cases}
		1, & i<k,\\
		\gamma, & i=k,\\
		0, & i>k,
	\end{cases}
	\]
	which shows that $w^\star$ has a threshold structure with at most one fractional component. 
    
	\smallskip
	\noindent\textbf{(Inner structure).}
	Given $q'^\star$, the inner problem is the linear program
	\[
	\max_{q\in\mathcal Q_\alpha(q'^\star)}
	\sum_{i=1}^n q_i\ell_i.
	\]
	Using the same loss ordering $\ell_1\le\cdots\le\ell_n$, suppose that there exist $i<j$ such that
	$q_i>0$ and $q_j<\frac{1}{\alpha}q_j'^\star$.	
    Then for any 
	$0<\delta\le\min\{q_i,\ \tfrac{1}{\alpha}q_j'^\star-q_j\}$,
	the perturbation
	$q_i\leftarrow q_i-\delta$, $q_j\leftarrow q_j+\delta$
	preserves feasibility and increases the objective by
	$\delta(\ell_j-\ell_i)\ge0$,
	with strict increase if $\ell_j>\ell_i$.
	Therefore, an optimal solution can be chosen that fills indices in decreasing loss order, including when losses are tied, up to the caps $q_i\le\frac{1}{\alpha}q_i'^\star$.
	Consequently, there exist $t\in[n]$ and $\eta\in[0,1]$ such that
	\[
	q_i^\star=
	\begin{cases}
		0, & i<t,\\
		\eta\,\frac{1}{\alpha}q_t'^\star, & i=t,\\
		\frac{1}{\alpha}q_i'^\star, & i>t,
	\end{cases}
	\]
	which shows that $q^\star$ has a threshold structure with at most one fractional boundary index.
\end{proof}

\bigskip

For uniform nominal weights, the two-stage solution takes the following explicit form.

\begin{corollary}[Exact trimming followed by exact worst-case allocation]
	\label{cor:exact_trimming_then_exact_alpha}
	Assume $\bar p_i = 1/n$ for all $i$, and reorder the samples so that
	\[
	\ell_1 \le \ell_2 \le \cdots \le \ell_n .
	\]
	Assume that $(1-\epsilon)n \in \mathbb{Z}$ and $\alpha(1-\epsilon)n \in \mathbb{Z}$.
	Let
	\[
	k_\epsilon := (1-\epsilon)n,
	\qquad
	m_\alpha := \alpha(1-\epsilon)n .
	\]
	Then there exists an optimal solution $(q'^\star,q^\star)$ such that
	\[
	q_i'^\star =
	\begin{cases}
		\dfrac{1}{(1-\epsilon)n}, & i=1,\dots,k_\epsilon,\\[6pt]
		0, & i>k_\epsilon,
	\end{cases}
	\]
	and
	\[
	q_i^\star =
	\begin{cases}
		0, & i=1,\dots,k_\epsilon-m_\alpha,\\[6pt]
		\dfrac{1}{\alpha(1-\epsilon)n}, & i=k_\epsilon-m_\alpha+1,\dots,k_\epsilon,\\[6pt]
		0, & i>k_\epsilon.
	\end{cases}
	\]
	In particular, $q'^\star$ is the exact $\epsilon$-trimmed empirical distribution,
	and $q^\star$ places uniform mass on the largest $\alpha(1-\epsilon)n$
	losses among the retained samples.
\end{corollary}

\begin{proof}
	By Theorem~\ref{thm:epsilon_then_alpha_structural}, there exists an optimal solution
	$(q'^\star,q^\star)$ such that the outer optimizer is obtained by filling the
	smallest-loss indices first. Then, we have
	\[
	q_i'^\star = \frac{\bar p_i w_i^\star}{1-\epsilon},
	\]
	where for some $k$ and $\gamma\in[0,1]$,
	\[
	w_i^\star=
	\begin{cases}
		1, & i<k,\\
		\gamma, & i=k,\\
		0, & i>k,
	\end{cases}
	\qquad
	\sum_{i=1}^n \bar p_i w_i^\star = 1-\epsilon.
	\]
	Since $\bar p_i=1/n$, the normalization condition becomes
	$$
	\sum_{i=1}^n \bar p_i w_i^\star = \frac{1}{n}\bigl((k-1)+\gamma\bigr)=1-\epsilon.
	$$
	Because $(1-\epsilon)n=k_\epsilon\in\mathbb Z$, it follows that
	$
	(k-1)+\gamma = k_\epsilon.
	$
	    Since $k_\epsilon \in \mathbb{Z}$ and $k_\epsilon = (k-1)+\gamma$ with $\gamma \in [0,1]$, we may take $\gamma=1$ and $k=k_\epsilon$, which also covers $\epsilon=0$.

    Then, we have
	\[
	q_i'^\star=
	\begin{cases}
		\dfrac{1}{(1-\epsilon)n}, & i=1,\dots,k_\epsilon,\\[6pt]
		0, & i>k_\epsilon.
	\end{cases}
	\]
	
	Now consider the inner maximization
	\[
	\max_{q\in\mathcal Q_\alpha(q'^\star)} \sum_{i=1}^n q_i\ell_i,
	\qquad
	0\le q_i\le \frac{1}{\alpha}q_i'^\star,
	\quad
	\sum_{i=1}^n q_i=1.
	\]
	Since $\ell_1\le\cdots\le\ell_n$, the objective is maximized by assigning as much
	mass as possible to the largest-loss indices. 
	Because the constraint $q_i \le \frac{1}{\alpha} q_i'^{\star}$ holds for every $i$,
	and $q_i'^{\star} = 0$ for all $i > k_\epsilon$, we have
	\[
	q_i^{\star} = 0
	\qquad
	\text{for all } i > k_\epsilon .
	\]
	
	For each retained index $i\le k_\epsilon$, we have
	\[
	\frac{1}{\alpha}q_i'^\star
	=
	\frac{1}{\alpha(1-\epsilon)n}
	=
	\frac{1}{m_\alpha},
	\]
	since $m_\alpha=\alpha(1-\epsilon)n$.
	Thus every non-zero upper bound is the same, namely $1/m_\alpha$.
	Therefore the maximizer is obtained by saturating exactly the largest
	$m_\alpha$ indices among $\{1,\dots,k_\epsilon\}$, namely
	$i=k_\epsilon-m_\alpha+1,\dots,k_\epsilon$.
	This yields
	\[
	q_i^\star=
	\begin{cases}
		0, & i=1,\dots,k_\epsilon-m_\alpha,\\[6pt]
		\dfrac{1}{\alpha(1-\epsilon)n}, & i=k_\epsilon-m_\alpha+1,\dots,k_\epsilon,\\[6pt]
		0, & i>k_\epsilon.
	\end{cases}
	\]
	Since there are exactly $m_\alpha$ such indices, the total mass is
	\[
	m_\alpha \cdot \frac{1}{\alpha(1-\epsilon)n}
	=
	\frac{m_\alpha}{m_\alpha}
	=
	1,
	\]
	so the simplex constraint is satisfied. This proves the claim.
\end{proof}

\paragraph{Remark.}
When $(1-\epsilon)n \notin \mathbb Z$, the loss-ordered outer solution has a fractional boundary atom at index $\lfloor(1-\epsilon)n\rfloor+1$.
The inner optimizer fills the retained atoms in decreasing loss order, with partial allocation at the cutoff when needed.

For uniform nominal weights, Algorithm~\ref{algo:additive_loss} computes the weights by sorting the losses, trimming the largest losses, and emphasizing the largest retained losses.
It assumes that $(1-\epsilon)n$ and $\alpha(1-\epsilon)n$ are integers; otherwise, fractional allocation is required at the boundaries as described above.

\begin{algorithm}[H]
	\caption{Procedure for minimizing the coupled DRO objective under \lossrevision{additive losses}}
	\label{algo:additive_loss}
	\begin{algorithmic}[1]
		\Require data $\{z_i\}_{i=1}^n$, parameters $\epsilon,\alpha$, step size $\eta$
		\State Initialize $\theta$
		\State Set $k_\epsilon \gets \lfloor (1-\epsilon)n \rfloor$ and $m_\alpha \gets \lfloor \alpha(1-\epsilon)n \rfloor$
		\While{not converged}
		
		\State Compute losses $\ell_i = \ell(z_i,\theta)$ for $i=1,\dots,n$
		\State Reorder the indices so that
		\[
		\ell_1 \le \ell_2 \le \cdots \le \ell_n
		\]

		\State Set
		\[
		q_i' \gets
		\begin{cases}
			\frac{1}{(1-\epsilon)n}, & i=1,\dots,k_\epsilon,\\
			0, & i>k_\epsilon
		\end{cases}
		\]
		
		\State Set
		\[
		q_i \gets
		\begin{cases}
			0, & i=1,\dots,k_\epsilon-m_\alpha,\\
			\frac{1}{\alpha(1-\epsilon)n}, & i=k_\epsilon-m_\alpha+1,\dots,k_\epsilon,\\
			0, & i>k_\epsilon
		\end{cases}
		\]
		
		\State Update
		\[
		\theta \gets \theta - \eta \nabla_\theta \left(\sum_{i=1}^n q_i\,\ell(z_i,\theta)\right)
		\]
		
		\EndWhile
		
		\State \Return $\theta$
	\end{algorithmic}
\end{algorithm}

\subsubsection{\lossrevision{Non-decomposable losses}} \label{appendix:non-decomposable_loss}

For non-decomposable losses, interactions across samples generally prevent the explicit weight allocation derived above.
We instead alternate between ascent in $q$, an update of $q'$ based on an estimated supergradient, and descent in $\theta$.
The main step is to characterize the sensitivity of the inner value function
\begin{align}
    \Phi_\alpha(q'; \theta)
    =
    \max_{q \in \mathcal{Q}_\alpha(q')} \mathcal L(q,\theta).
    \label{eq:inner_maximization_current}
\end{align}
Proposition~\ref{prop:outer_supergrad_main} relates this sensitivity to the upper-bound KKT multipliers under local stability~\citep{rockafellar1998variational}.
Theorem~\ref{thm:outer_supergradient} gives their explicit form.

\bigskip
\begin{outerproprestated}[restated]
Fix $\theta$, $0<\alpha<1$, and $q'\in\mathcal S_{n-1}$ with $q'_i>0$ for all $i$.
Set $b=q'/\alpha$, $\mathcal F(b)=\{q\in\mathcal S_{n-1}:q\le b\}$, and
\[
V(b):=\max_{q\in\mathcal F(b)}\mathcal L(q,\theta),
\qquad
\Phi_\alpha(q';\theta)=V(b).
\]
Let $q^\star(b)$ be a global maximizer, and assume that $\mathcal L(\cdot,\theta)$ is continuous on the simplex and continuously differentiable near $q^\star(b)$.
Assume local stability: for every sufficiently small $\Delta b$, a global maximizer $q^\star(b+\Delta b)$ can be selected such that
\begin{equation}
\|q^\star(b+\Delta b)-q^\star(b)\|\le C\|\Delta b\|
\label{eq:inner_solution_stability}
\end{equation}
for some constant $C>0$ independent of $\Delta b$.
Then any KKT multiplier $\mu^\star$ for the upper-bound constraints $q_i\le b_i$ at $q^\star(b)$ satisfies
\[
\frac{\mu^\star}{\alpha}\in\partial^+\Phi_\alpha(q';\theta),
\]
where $\partial^+$ denotes the superdifferential: $s\in\partial^+F(u)$ if $F(u+h)\le F(u)+\langle s,h\rangle+o(\|h\|)$ as $h\to0$ within the domain of $F$.
\end{outerproprestated}

\begin{proof}
We prove the result by combining the KKT conditions with a first-order Taylor expansion.
The local stability assumption controls the remainder under perturbations of $b$.

Write $\Delta q=q^\star(b+\Delta b)-q^\star(b)$ and $g^\star=\nabla_q\mathcal L(q^\star(b),\theta)$.
The KKT conditions give $g^\star=\nu^\star\mathbf 1-\lambda^\star+\mu^\star$, where $\lambda^\star$ and $\mu^\star$ are the nonnegative lower- and upper-bound multipliers.
Since both solutions sum to one, feasibility and complementary slackness imply
\[
\begin{aligned}
\langle g^\star,\Delta q\rangle
&=-\langle\lambda^\star,\Delta q\rangle+\langle\mu^\star,\Delta q\rangle\\
&\le\langle\mu^\star,\Delta b\rangle.
\end{aligned}
\]

By~\eqref{eq:inner_solution_stability}, $\|\Delta q\|=O(\|\Delta b\|)$.
A first-order Taylor expansion therefore yields
\[
\begin{aligned}
V(b+\Delta b)-V(b)
&=\langle g^\star,\Delta q\rangle+o(\|\Delta b\|)\\
&\le\langle\mu^\star,\Delta b\rangle+o(\|\Delta b\|).
\end{aligned}
\]
Thus $\mu^\star\in\partial^+V(b)$.
Substituting $b=q'/\alpha$ gives $\mu^\star/\alpha\in\partial^+\Phi_\alpha(q';\theta)$.
\end{proof}

\paragraph{Applying the supergradient result to weighted Cox.}
The following lemma gives conditions under which the weighted Cox loss satisfies Proposition~\ref{prop:outer_supergrad_main}.
The key step is to bound changes in the inner optimizer by changes in its upper bounds.

For fixed $\theta$ with $f_\theta(x_j)\in\mathbb R$ for all $j$, write $a_j=\exp(f_\theta(x_j))>0$ and $S_i(q)=\sum_{j:y_j\ge y_i}a_jq_j$.
The weighted Cox objective is $\mathcal L(q,\theta)=\sum_i\delta_iq_i\{\log S_i(q)-\log a_i\}$, with zero-mass terms interpreted by continuity.

\begin{lemma}[Sufficient conditions for weighted Cox]
\label{lem:weighted_cox_stability}
Let $0<\alpha<1$, $q'\in\mathcal S_{n-1}$ with $q'_i>0$ for all $i$, and $b=q'/\alpha$.
Suppose the weighted Cox objective has a unique global maximizer $q^\star$ over $\mathcal F(b)$, $S_i(q^\star)>0$ for every observed event $i$ with $\delta_i=1$, and the coordinates of $g^\star=\nabla_q\mathcal L(q^\star,\theta)$ are pairwise distinct.
Then the objective is continuous on the simplex and continuously differentiable near $q^\star$, and the global inner optimizer satisfies the local stability condition~\eqref{eq:inner_solution_stability}.
Consequently, the supergradient characterization in Proposition~\ref{prop:outer_supergrad_main} applies.
\end{lemma}

\begin{proof}
We first establish continuity of $\mathcal L$, continuous differentiability near $q^\star$, and convergence of global maximizers as the upper bounds approach $b$.
We then use the strict ordering of the gradient coordinates to derive an explicit representation of these maximizers and establish local Lipschitz stability.

Since $S_i(q)\ge a_iq_i$, the term $q_i\log S_i(q)$ tends to zero whenever $S_i(q)\to0$ on the simplex.
This gives a continuous extension of the objective to the simplex.
Positivity of all event risk-set sums at $q^\star$ ensures continuous differentiability in a neighborhood of $q^\star$.

The point $q'$ is strictly feasible relative to the simplex because $0<q'_i<b_i$.
Let $d=\min_i(b_i-q'_i)>0$.
For sufficiently small $\Delta b$, the point $(1-t)q^\star+tq'$, where $t=\|\Delta b\|_\infty/d$, belongs to $\mathcal F(b+\Delta b)$ and converges to $q^\star$.
By compactness and continuity, any limit of global maximizers for nearby bounds is therefore a global maximizer at $b$.
Uniqueness implies that all such maximizers converge to $q^\star$.

Choose a permutation $\pi$ such that $g^\star_{\pi(1)}>\cdots>g^\star_{\pi(n)}$.
Continuity of the gradient ensures that this ordering persists at every nearby global maximizer $\widehat q$.
The first-order necessary condition on the convex set $\mathcal F(b+\Delta b)$ gives
\[
\langle\nabla_q\mathcal L(\widehat q,\theta),q-\widehat q\rangle\le0,
\qquad q\in\mathcal F(b+\Delta b).
\]
Thus $\widehat q$ maximizes its linearized objective over this set.
With strictly ordered coefficients, this linear problem has the unique solution
\[
\widehat q_{\pi(k)}
=\min\left\{
b_{\pi(k)}+\Delta b_{\pi(k)},\;
\left[1-\sum_{j<k}(b_{\pi(j)}+\Delta b_{\pi(j)})\right]_+
\right\},\qquad k\in[n].
\]
This expression is a continuous piecewise affine function of $b+\Delta b$, and hence is locally Lipschitz.
It equals $q^\star$ at $\Delta b=0$, proving~\eqref{eq:inner_solution_stability}.
\end{proof}

\bigskip
We next express the multipliers in terms of the loss gradient.
\bigskip
\begin{theorem}[Outer supergradient characterization]
\label{thm:outer_supergradient}

Fix $\theta \in \Theta$ and $q' \in \mathcal{S}_{n-1}$. 
Let $\Phi_\alpha(q';\theta)$ be the inner value function
defined in \eqref{eq:inner_maximization_current}, where
$\mathcal L(\cdot,\theta) : \mathcal{S}_{n-1} \to \mathbb{R}$ is continuous.

Assume that the inner maximization problem
\[
\max_q \mathcal{L}(q,\theta)
\quad
\text{s.t.}\quad
\sum_{i=1}^n q_i=1,\;
q_i\ge0,\;
q_i\le \tfrac{1}{\alpha}q'_i
\]
admits an optimal solution $q^\star$, and that a constraint qualification holds so that the KKT conditions are satisfied with multipliers
\[
\nu^\star \in \mathbb{R}, \quad 
\lambda^\star \in \mathbb{R}^n_+, \quad 
\mu^\star \in \mathbb{R}^n_+,
\]
where $\nu^\star$, $\lambda^\star$, and $\mu^\star$ are the Lagrange multipliers associated with the constraints
$\sum_{i=1}^n q_i = 1$, $q_i \ge 0$, and $q_i \le \tfrac{1}{\alpha} q'_i$, respectively.
Assume also that $\mathcal L(\cdot,\theta)$ is continuously differentiable on a neighborhood of $q^\star$.

Let $g^\star := \nabla_q \mathcal L(q^\star, \theta)$. 
Then the following statements hold:

\medskip
\noindent
\textbf{(i) Stationarity and complementary slackness}
\[
g^\star - \nu^\star \mathbf{1} + \lambda^\star - \mu^\star = 0,
\]
\[
\lambda^\star_i q^\star_i = 0,
\qquad
\mu^\star_i \Big(q^\star_i - \tfrac{1}{\alpha} q'_i\Big) = 0,
\qquad
\forall i.
\]

\medskip
\noindent
\textbf{(ii) Interior characterization}
For any $i$ such that $0 < q_i^\star < \tfrac{1}{\alpha} q'_i$,
\[
\nu^\star = g_i^\star.
\]

\medskip
\noindent
\textbf{(iii) Multiplier characterization}
For each nondegenerate coordinate with $q'_i>0$,
\[
\mu^\star_i =
\begin{cases}
(g^\star_i - \nu^\star)_+, & \text{if } q^\star_i = \tfrac{1}{\alpha} q'_i,\\
0, & \text{if } q^\star_i < \tfrac{1}{\alpha} q'_i.
\end{cases}
\]

\medskip
\noindent
\textbf{(iv) Supergradient}
If, in addition, $0<\alpha<1$, $q'_i>0$ for all $i$, and the local stability condition~\eqref{eq:inner_solution_stability} holds at $q^\star=q^\star(b)$ with $b=q'/\alpha$, then
\[
\frac{1}{\alpha}\mu^\star \in \partial^+ \Phi_\alpha(q';\theta).
\]
\end{theorem}
\bigskip
\begin{proof}
We first use stationarity and complementary slackness to identify the upper-bound multipliers.
Applying Proposition~\ref{prop:outer_supergrad_main} under the additional assumptions in part (iv) then gives the supergradient characterization.

	Set
	\[
	b := \frac{1}{\alpha} q',
	\qquad
	\mathcal F(b)
	:=
	\left\{
	q \in \mathbb{R}^n :
	\sum_{i=1}^n q_i = 1,\ q \ge 0,\ q \le b
	\right\}.
	\]
	By the definition of $\Phi_\alpha$,
	\[
	\Phi_\alpha(q';\theta)
	=
	\max_{q \in \mathcal F(b)} \mathcal L(q,\theta).
	\]

By assumption, an optimal solution exists and a constraint qualification holds, and hence the KKT conditions are satisfied.

	\paragraph{Step 1: KKT conditions.}
	Let $q^\star$ be an optimal solution of the inner problem, and let
	\[
	\nu^\star \in \mathbb{R},
	\qquad
	\lambda^\star \in \mathbb{R}^n_+,
	\qquad
	\mu^\star \in \mathbb{R}^n_+
	\]
	be the corresponding Lagrange multipliers associated with the constraints
	$\sum_{i=1}^n q_i = 1$, $q_i \ge 0$, and $q_i \le b_i$, respectively.
	Define
	\[
	g^\star := \nabla_q \mathcal L(q^\star,\theta).
	\]
	Then the stationarity condition is
	\[
	g^\star - \nu^\star \mathbf{1} + \lambda^\star - \mu^\star = 0,
	\]
	where $\mathbf{1} \in \mathbb{R}^n$ denotes the vector of all ones.
	The complementary slackness conditions are
	\[
	\lambda_i^\star q_i^\star = 0,
	\qquad
	\mu_i^\star (q_i^\star - b_i) = 0,
	\qquad
	\forall i \in [n].
	\]

\paragraph{Step 2: Identification of the multipliers for the upper-bound constraints.}
Consider an index $i$ such that $0 < q_i^\star < b_i$.
Then both the lower-bound and upper-bound constraints are inactive, so
\[
\lambda_i^\star = 0,
\qquad
\mu_i^\star = 0.
\]
Hence, by the stationarity condition,
\[
g_i^\star = \nu^\star.
\]

If $q_i^\star < b_i$, then complementary slackness gives
\[
\mu_i^\star = 0.
\]
If $0 < q_i^\star = b_i$, then the lower-bound constraint is inactive, so
\[
\lambda_i^\star = 0.
\]
The stationarity condition then gives
\[
\mu_i^\star = g_i^\star - \nu^\star.
\]
Since $\mu_i^\star \ge 0$, we obtain
\[
\mu_i^\star =
\begin{cases}
(g_i^\star - \nu^\star)_+, & \text{if } 0 < q_i^\star = b_i,\\
0, & \text{if } q^\star_i < b_i.
\end{cases}
\]
The degenerate case $q_i^\star=b_i=0$ is not included in this characterization.
In this case, the lower-bound and upper-bound constraints are both active, and the multipliers $\lambda_i^\star$ and $\mu_i^\star$ are not separately identifiable from the stationarity condition alone.

\paragraph{Step 3: Supergradient characterization.}
Under the additional assumptions in part (iv), Proposition~\ref{prop:outer_supergrad_main} gives
\[
\frac{1}{\alpha}\mu^\star\in\partial^+\Phi_\alpha(q';\theta),
\]
which completes the proof.
\end{proof}

\bigskip
\paragraph{Supergradient estimation.}
At an exact inner maximizer, Theorem~\ref{thm:outer_supergradient}(ii) gives $g_i^\star=\nu^\star$ on interior coordinates.
For an approximate maximizer $\hat q$, let $\hat g=\nabla_q\mathcal L(\hat q,\theta)$ and estimate this common value by
\[
\nu := \frac{1}{|\mathcal I|}\sum_{i\in\mathcal I} \hat g_i,
\qquad
\mathcal I
=
\left\{
i \in [n] :
0 < \hat q_i < \tfrac{1}{\alpha} q'_i
\right\}.
\]
Substituting $\hat q$, $\hat g$, and $\nu$ into the multiplier formula gives the supergradient estimate used in the $q'$-update.
The exact characterization requires the assumptions of part (iv) at an inner optimum.
If $\mathcal I$ is empty, $q'$ is left unchanged.
We use small positive lower bounds on sample weights and weighted risk-set sums for numerical stability.

Theorem~\ref{thm:local_convergence} below establishes local convergence of the outer variables $(q',\theta)$ under exact inner maximization and the stated assumptions.
For computational efficiency, we approximate the inner maximizer using a finite number of gradient steps and estimate the KKT multipliers from the resulting iterate for the outer update.

\subsubsection{Computational complexity}
\label{appendix:computational_complexity}
Let $n$ be the sample size, $d$ the covariate dimension, $T$ the number of inner weight updates, and $K$ the number of outer iterations in Algorithm~\ref{alg:alternating_nested_dro}.
For the linear Cox model, evaluating the risk scores and their parameter derivatives requires $O(nd)$ time.
Weighted risk-set sums are computed cumulatively after sorting observations by time.
The two capped-simplex projections can be computed exactly in $O(n\log n)$ time by sorting; our implementation finds the multiplier $\lambda$ by bisection, which costs $O(n)$ per bisection step.
Consequently, the total computational complexity is
\[
O\!\left(KTn(d+\log n)\right).
\]
For fixed $K$, $T$, and $d$, this cost is nearly linear in the sample size.
The additional memory required for the sample weights and cumulative risk-set quantities is $O(n)$, beyond storage for the data and model parameters.
Empirical model-fitting times are reported in Appendix~\ref{appendix:runtime}.

\subsubsection{Local convergence analysis}
\label{appendix:convergence}
We analyze local convergence of the outer variables $q'$ and $\theta$, assuming that before each $q'$-update, we obtain the exact global inner maximizer $q^\star$ given the current $q'$ and $\theta$.
Under the assumptions below, we show that the outer iterates of the specified alternating scheme with KL-projected $q'$-updates and gradient updates of $\theta$ converge R-linearly in the Euclidean norm to a strict local minimizer of the resulting outer value function, for a sufficiently small common step size and sufficiently close feasible initialization.

Write $x=(q',\theta)$, $\mathcal C=\mathcal Q_{1-\epsilon}(\bar p)$, and $F(x)=\Phi_\alpha(q';\theta)$.
Let $q^\star(x)$ denote the global inner maximizer and $\mu^\star(x)$ its upper-bound KKT multipliers.
With $D_{\mathrm{KL}}(u\|v)=\sum_i u_i\log(u_i/v_i)$, the outer weight update takes the entropic mirror-descent form of \citet{beck2003mirror}:
\[
q'^{+}=\arg\min_{u\in\mathcal C}
\left\{\left\langle\frac{\mu^\star(x)}{\alpha},u-q'\right\rangle
+\frac{1}{\eta}D_{\mathrm{KL}}(u\|q')\right\}.
\]
This is equivalent to the exponentiated update~\eqref{eq:q'_update} with KL projection.
We retain the update order of Algorithm~\ref{alg:alt_nested_dro}: $T$ weight steps with $\theta$ fixed, followed by a parameter step using the same inner solution as the last weight step.
For the analysis, set $\eta_{q'}=\eta_\theta=\eta$ and fix $T\ge1$.

\bigskip
\begin{theorem}[Local convergence]
\label{thm:local_convergence}
Let $0<\alpha<1$, $0<\epsilon<1$, and $x^\star=(q'^{\star},\theta^\star)$ be feasible, with $q'^{\star}>0$ and $\theta^\star$ interior to the parameter space.
Assume:
\begin{enumerate}[label=(\roman*),leftmargin=*,nosep]
\item \textbf{Smooth inner solution.}
$\mathcal L$ is twice continuously differentiable near $(q^\star(x^\star),\theta^\star)$.
For every $x$ near $x^\star$, the global inner maximizer $q^\star(x)$ is unique, at least one coordinate satisfies $0<q_i^\star(x)<q'_i/\alpha$, and both $q^\star(x)$ and $\mu^\star(x)$ are continuously differentiable in $x$.
\item \textbf{Outer stationarity.}
The outer KKT conditions hold at $x^\star$, and every active upper bound of $q'$ has a strictly positive outer KKT multiplier.
\item \textbf{Positive local curvature.}
At $x^\star$, the Hessian of $F$ is positive definite along directions that preserve the simplex constraint and keep the active bound coordinates of $q'$ fixed, with $\theta$ free to vary.
\end{enumerate}
Then $(q'^{\star},\theta^\star)$ is a strict local minimizer of the outer value function $F$.
For a sufficiently small fixed $\eta>0$ and feasible initialization sufficiently close to $(q'^{\star},\theta^\star)$, there exist constants $C>0$ and $\rho\in(0,1)$ such that the outer iterates of the specified KL-projected alternating scheme satisfy
\[
\left\|
\begin{pmatrix}
q'^{(k)}-q'^{\star}\\
\theta^{(k)}-\theta^\star
\end{pmatrix}
\right\|_2
\le C\rho^k,
\qquad k=0,1,2,\ldots.
\]
\end{theorem}

\begin{proof}
We first show that $F$ is smooth and $x^\star$ is a strict local minimizer.
Next, we show that the KL update fixes the weight coordinates whose upper bounds are active at $x^\star$.
With these coordinates fixed, we prove that a full update cycle contracts the error in a weighted norm, yielding the stated Euclidean bound.

\paragraph{Step 1: Smoothness and local minimality.}
By assumption (i), the inner optimizer varies smoothly with $q'$ and $\theta$.
Differentiating the optimized loss and applying the inner KKT conditions gives, for perturbations along the simplex,
\[
\mathrm dF
=\left\langle\frac{\mu^\star(x)}{\alpha},\mathrm dq'\right\rangle
+\left\langle\nabla_\theta\mathcal L(q^\star(x),\theta),\mathrm d\theta\right\rangle.
\]
The derivatives in this expression are continuously differentiable by assumption (i), so $F$ is twice continuously differentiable locally.
Outer stationarity and strict complementarity in (ii), together with the positive curvature in (iii), give the second-order sufficient conditions for $x^\star$ to be a strict local minimizer.

\paragraph{Step 2: Identification of active upper bounds.}
We next show that nearby KL updates place the same coordinates of $q'$ at their upper bounds as at $x^\star$.
We adapt the active-constraint argument of \citet{bertsekas1976goldstein} to the KL update.
At $x^\star$, its subproblem has a positive definite objective Hessian, independent active constraint gradients, and strictly positive active multipliers by (ii).
Its solution therefore varies smoothly with nearby inputs while retaining the same active upper bounds.
These weight coordinates can thus be held fixed when analyzing subsequent nearby updates.

\paragraph{Step 3: Contraction and the Euclidean error bound.}
With the active coordinates fixed, we analyze the remaining weight coordinates jointly with $\theta$.
Let $U$ have orthonormal columns spanning the directions that preserve $\sum_i q'_i=1$ and keep $q'_i$ fixed whenever $q_i'^{\star}=\bar p_i/(1-\epsilon)$.
Write $q'=q'^{\star}+Uz$, so $z$ describes changes in the remaining weight coordinates.
By assumption (iii), the Hessian $H$ of $F(q'^{\star}+Uz,\theta)$ at $(0,\theta^\star)$ is positive definite.
Set
\[
M_q=U^\top\operatorname{diag}(1/q'^{\star})U,
\qquad M=\operatorname{diag}(M_q,I),
\qquad D=\operatorname{diag}(M_q^{-1},0).
\]
Zero-dimensional weight blocks are omitted.
Differentiating the KL optimality equations gives $R=I-\eta DH$ for a weight-only step and $J=I-\eta M^{-1}H$ for the final weight and parameter step.
The derivative of a full cycle is therefore $JR^{T-1}$.

For sufficiently small $\eta>0$, $H\succ0$ and $0\preceq D\preceq M^{-1}$ imply $\|J\|_H<1$ and $\|R\|_H\le1$, where $\|v\|_H^2=v^\top Hv$ and the corresponding matrix norm is induced by this vector norm.
By continuity of the cycle map's derivative, the full cycle is therefore a local contraction in this norm after these upper-bound coordinates have been fixed.
Iterating the contraction gives geometric decay of the error in the $H$-norm.
Norm equivalence and the orthonormality of $U$ yield the stated Euclidean bound $C\rho^k$, with all fixed factors and the finite initial segment absorbed into $C$.
Continuity of $q^\star(x)$ also gives convergence of the exact inner maximizers.
\end{proof}

\clearpage
\subsubsection{Alternating gradient algorithm}
\label{appendix:non_decomposable_algorithm}

Algorithm~\ref{alg:alternating_nested_dro} alternates updates of the adversarial weights $q$ and nominal weights $q'$, followed by a model-parameter update.

\begin{algorithm}[H]
\caption{Alternating gradient algorithm for the coupled DRO problem}
\label{alg:alternating_nested_dro}
\begin{algorithmic}[1]

\State \textbf{Input:} step sizes $\eta_q, \eta_{q'}, \eta_\theta$; update steps $T$, $K$; initial values $\theta^{(0)}, q^{(0)}, q'^{(0)}$

\For{$k = 0,1,2,\ldots, K-1$}
    \State Set $q^{(k,0)} \leftarrow q^{(k)}$, \; $q'^{(k,0)} \leftarrow q'^{(k)}$
    
    \Statex \Comment{Update $(q, q')$}
    \For{$t = 0,\dots,T-1$}
    
        \State Compute $g^{(k,t)} \leftarrow \nabla_q \mathcal{L}(q^{(k,t)}, \theta^{(k)})$
        \State Set
        \[
        \tilde q^{(k,t+1)}_i \leftarrow
        \frac{q^{(k,t)}_i \exp\left(\eta_q g^{(k,t)}_i\right)}
        {\sum_{j=1}^n q^{(k,t)}_j \exp\left(\eta_q g^{(k,t)}_j\right)}
        \]
        \State Project
        \[
        q^{(k,t+1)} \leftarrow \Pi_{\mathcal{Q}_\alpha(q'^{(k,t)})}(\tilde q^{(k,t+1)})
        \]
        
        \State Compute $g^{(k,t+1)} \leftarrow \nabla_q \mathcal{L}(q^{(k,t+1)}, \theta^{(k)})$
        \State Define
        \[
        \mathcal I^{(k,t+1)} \leftarrow
        \left\{ i \in [n] : 0 < q^{(k,t+1)}_i < \tfrac{1}{\alpha} q'^{(k,t)}_i \right\}
        \]
        
        \If{$\mathcal I^{(k,t+1)} = \emptyset$}
            \State Set $q'^{(k,t+1)} \leftarrow q'^{(k,t)}$
        \Else
            \State Compute
            \[
            \nu^{(k,t+1)} \leftarrow
            \frac{1}{|\mathcal I^{(k,t+1)}|}
            \sum_{i\in\mathcal I^{(k,t+1)}} g^{(k,t+1)}_i
            \]
            
            \State Set
            \[
            s^{(k,t+1)}_i \leftarrow
            \begin{cases}
            \frac{1}{\alpha}(g^{(k,t+1)}_i - \nu^{(k,t+1)})_+, 
            & \text{if } 0 < q^{(k,t+1)}_i = \tfrac{1}{\alpha} q'^{(k,t)}_i,\\
            0, & \text{otherwise}.
            \end{cases}
            \]
            
            \State Set
            \[
            \tilde q'^{(k,t+1)}_i \leftarrow
            \frac{q'^{(k,t)}_i \exp\left(-\eta_{q'} s^{(k,t+1)}_i\right)}
            {\sum_{j=1}^n q'^{(k,t)}_j \exp\left(-\eta_{q'} s^{(k,t+1)}_j\right)}
            \]
            \State Project
            \[
            q'^{(k,t+1)} \leftarrow \Pi_{\mathcal{Q}_{1-\epsilon}(\bar p)}(\tilde q'^{(k,t+1)})
            \]
        \EndIf
        
    \EndFor
    
    \State Set $q^{(k+1)} \leftarrow q^{(k,T)}$, \; $q'^{(k+1)} \leftarrow q'^{(k,T)}$
    
    \Statex \Comment{Update $\theta$}
    \State Update
    \[
    \theta^{(k+1)} \leftarrow
    \theta^{(k)} - \eta_\theta \nabla_\theta \mathcal{L}(q^{(k+1)}, \theta^{(k)})
    \]
    
\EndFor

\end{algorithmic}
\end{algorithm}
\clearpage

\subsection{Sensitivity analysis}\label{appendix:sensitivity_analysis}
\subsubsection{Robustness parameters}
\label{appendix:robustness_parameters}

We examine sensitivity to $\alpha$ and $\epsilon$, which control robustness to subpopulation shift and outlier contamination, respectively.
We consider $\alpha \in \{0.05, 0.2, 0.4, 1\}$ and $\epsilon \in \{0, 0.01, 0.05, 0.1\}$. 
When varying $\alpha$, we fix $\epsilon=0$; when varying $\epsilon$, we fix $\alpha=0.1$.

We use the data-generating process in Section~\ref{ssec:simulation} with adjusted mixture proportions and $N_{\text{train}}=N_{\text{test}}=500$.
The training distribution is $P_{\rm train}=0.95P_1+0.05P_2$ for subpopulation shift and $\widetilde P_{\rm train}=0.7P_1+0.2P_2+0.1\widetilde P$ for outlier contamination.
Both settings use the clean test distribution $P_{\rm test}=0.8P_1+0.2P_2$.

\begin{figure}[H]
\centering
\includegraphics[width=0.49\linewidth]{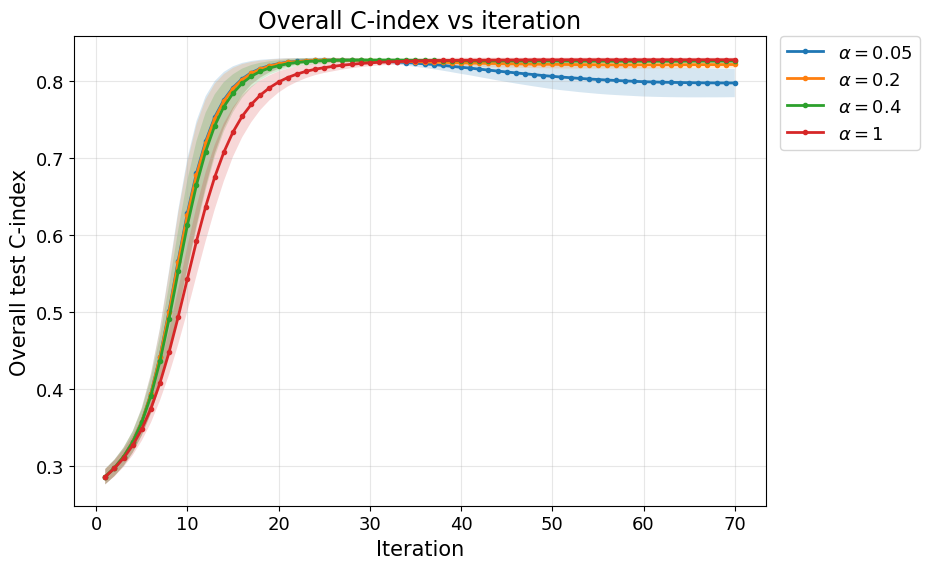}
\includegraphics[width=0.49\linewidth]{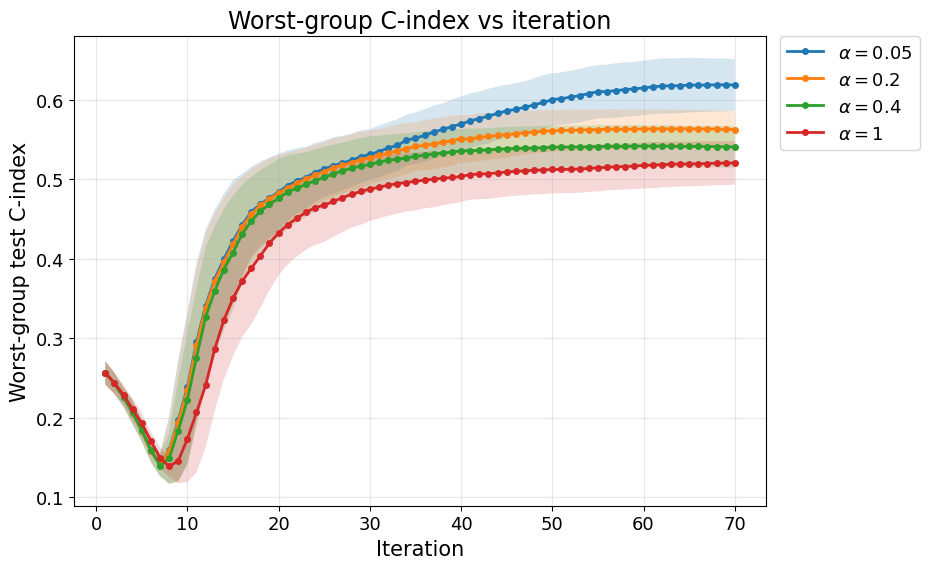}
\includegraphics[width=0.49\linewidth]{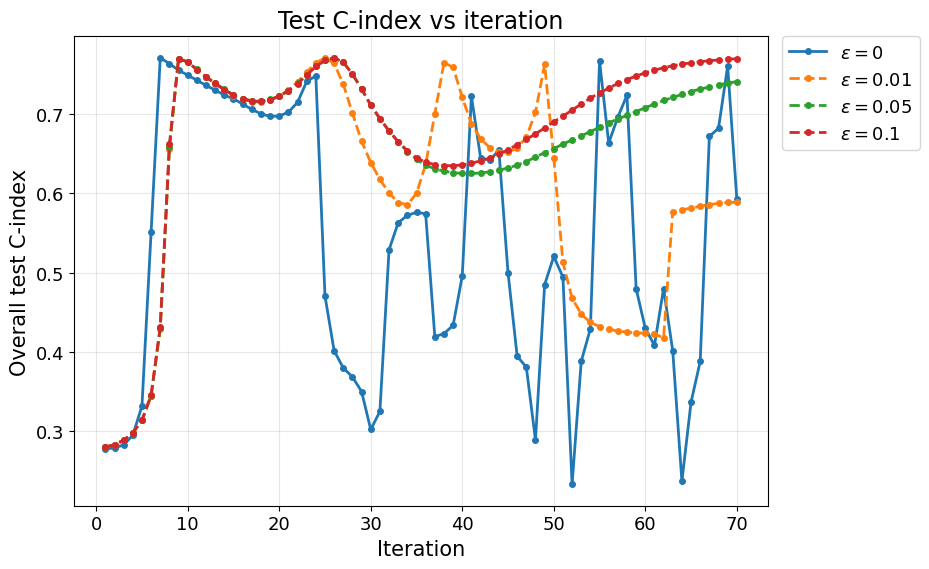}
\caption{Sensitivity to $\alpha$ and $\epsilon$. (Top-left and top-right) Overall and worst-group test C-index for varying $\alpha$ under subpopulation shift. (Bottom) Overall test C-index for varying $\epsilon$ under outlier contamination.}
\label{fig:sensitivity}
\end{figure}

In Figure~\ref{fig:sensitivity}, smaller $\alpha$ improves worst-group C-index, with a slight decrease in overall C-index at $\alpha=0.05$.
Under contamination, positive $\epsilon$ reduces training oscillations, with smoother trajectories at $\epsilon=0.05$ and $0.1$.

\subsection{Additional experiments and experimental details}
\label{appendix:experiments_details}
We examine robustness in additional simulated and real-data settings and describe the experimental procedures.
All experiments were conducted on a server equipped with Intel Xeon Gold 6230R CPUs and NVIDIA GeForce RTX 3090 GPUs.

\subsubsection{Simulation: Subpopulation shift}\label{appendix:sim_subpopulation_shift}
We isolate subpopulation shift by using uncontaminated mixtures of two latent groups.
The training distribution is
$P_{\rm train} = \pi^{\mathrm{tr}}_1 P_1 + \pi^{\mathrm{tr}}_2 P_2$,
with $(\pi^{\mathrm{tr}}_1, \pi^{\mathrm{tr}}_2) = (0.95, 0.05)$ and $N_{\text{train}} = 3{,}000$. 
The test distribution is
$P_{\rm test} = \pi^{\mathrm{te}}_1 P_1 + \pi^{\mathrm{te}}_2 P_2$,
with $(\pi^{\mathrm{te}}_1, \pi^{\mathrm{te}}_2) = (0.8, 0.2)$ and $N_{\text{test}} = 5{,}000$. 
The remaining data-generation settings follow Section~\ref{ssec:simulation}.
\begin{figure}[H]
\centering
\includegraphics[width=0.9\linewidth]{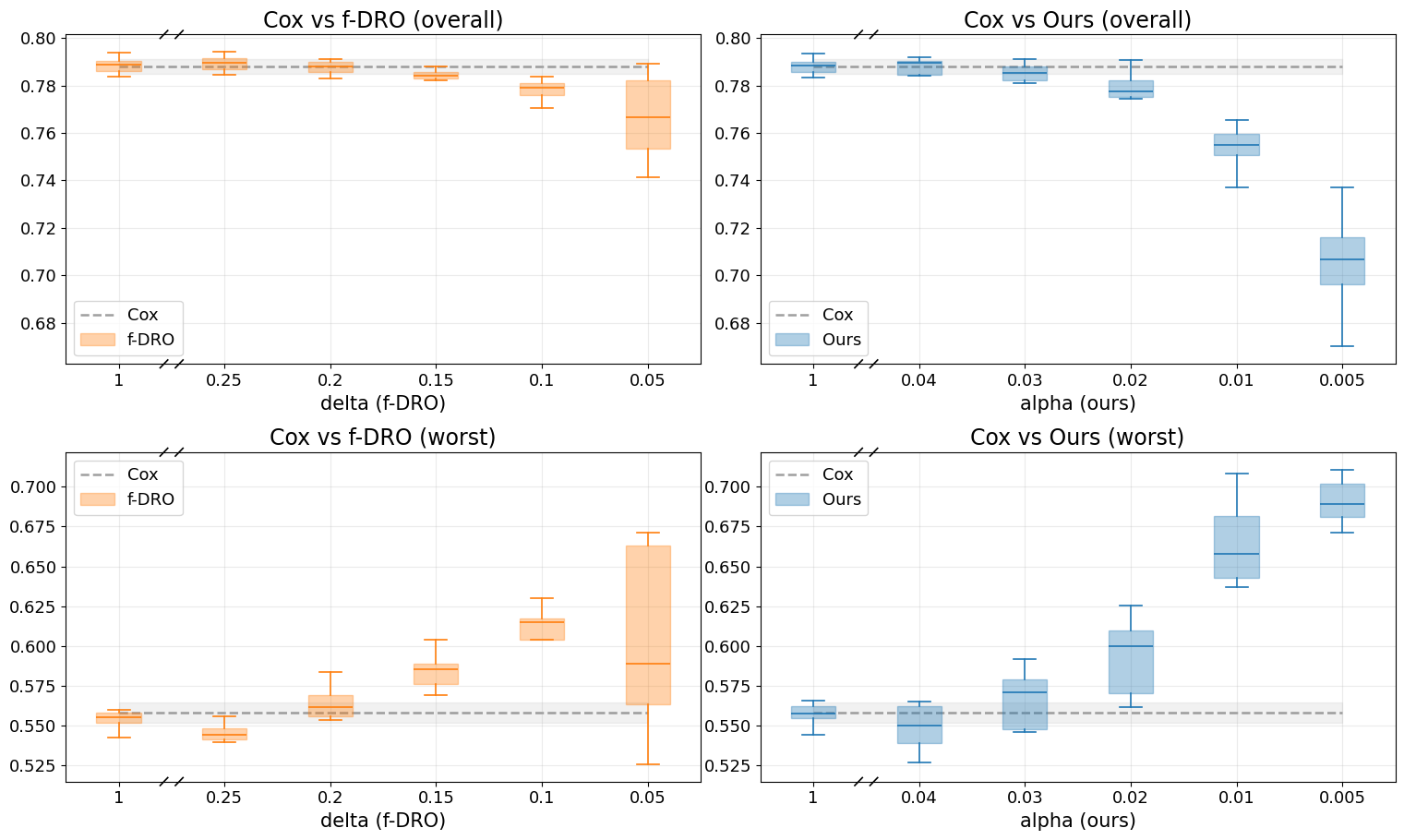}
\caption{Overall and worst-group C-index under subpopulation shift for the Cox model, $f$-DRO (left panels), and the proposed method (right panels), across robustness parameters. Box plots summarize 20 independent runs.}
\label{fig:subpopulation_shift_box_plot}
\end{figure}

Figure~\ref{fig:subpopulation_shift_box_plot} shows that both robust methods can improve worst-group C-index at a cost in overall C-index, depending on the robustness parameters.

\subsubsection{Simulation: Nonlinear risk functions}
\label{appendix:nonlinear_simulation}
We extend the simulation in Section~\ref{ssec:simulation} to nonlinear group-specific risk functions with linear, quadratic, and interaction terms.
We index the majority, minority, and outlier components by $g=0,1,2$, respectively, and define
\[
\begin{aligned}
r_0(x) &= x_1+x_2+0.15x_1^2+0.15x_2^2+0.15x_1x_2,\\
r_1(x) &= x_1-x_2+0.15x_1^2-0.15x_2^2-0.15x_1x_2,\\
r_2(x) &= -3x_1-3x_2-0.45x_1^2-0.45x_2^2-0.45x_1x_2.
\end{aligned}
\]
Survival times follow the group-specific hazard $h_g(t\mid x)=h_0\exp(r_g(x))$.
The training mixture proportions are $(0.90,0.08,0.02)$, and the clean test proportions are $(0.80,0.20,0)$.
We use a single-hidden-layer neural Cox model with width 5; the training and test sample sizes, covariate generation, and censoring setup are otherwise the same as in Section~\ref{ssec:simulation}.

\paragraph{Training trajectories.}
Figure~\ref{fig:nonlinear_cindex_trajectory_summary} summarizes the test C-index trajectories, including the substantial variation of $f$-DRO and the proposed method with $\epsilon=0$.
With $\epsilon=0.05$, the proposed method attains higher late-stage worst-group C-index than Cox-NN and $f$-DRO, with lower overall C-index than Cox-NN.
The corresponding IBS trajectories are provided in Appendix~\ref{appendix:ibs_trajectories}.

\begin{figure}[H]
\centering
\includegraphics[width=\textwidth]{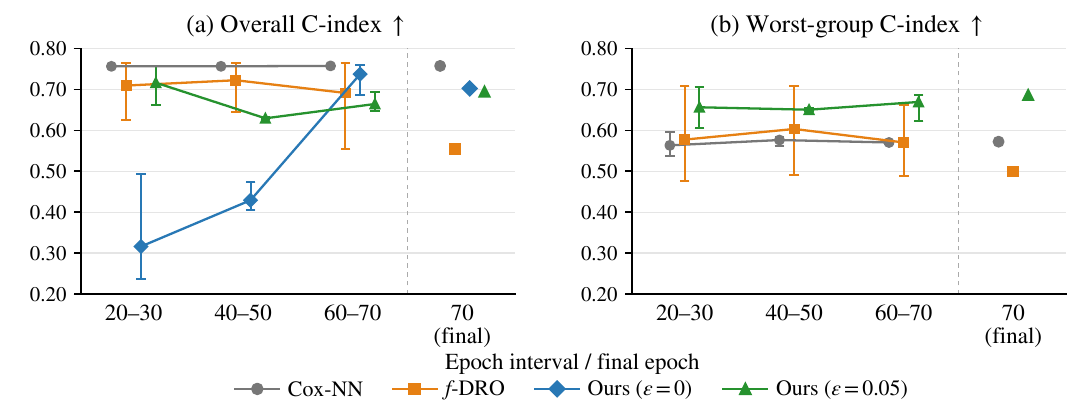}
\caption{Test C-index under nonlinear risks and outlier contamination. Points show interval means, whiskers show minima and maxima over epochs, and lines connect means. Epoch 70 is shown separately.}
\label{fig:nonlinear_cindex_trajectory_summary}
\end{figure}

\paragraph{Repeated-run performance.}
Table~\ref{tab:nonlinear_repeated_results} reports a separate experiment averaged over five independent repetitions.
The proposed method with $\epsilon=0.05$ outperforms $f$-DRO on all four metrics and improves worst-group C-index and IBS relative to Cox-NN, at a cost in overall performance.

\begin{table}[H]
\caption{Test performance under nonlinear risks and outlier contamination, averaged over five independent repetitions. Bold and underlined entries denote the best and second-best results, respectively.}
\label{tab:nonlinear_repeated_results}
\centering
\small
\renewcommand{\arraystretch}{0.95}
\begin{tabularx}{0.85\textwidth}{
  >{\centering\arraybackslash}l
  >{\centering\arraybackslash}X
  >{\centering\arraybackslash}X
  >{\centering\arraybackslash}X
  >{\centering\arraybackslash}X
}
\toprule
\textbf{Method}
& \textbf{Overall C-index} $\uparrow$
& \textbf{Worst-group C-index} $\uparrow$
& \textbf{Overall IBS} $\downarrow$
& \textbf{Worst-group IBS} $\downarrow$
\\
\midrule
Cox-NN
& \textbf{0.753}
& \underline{0.562}
& \textbf{0.125}
& 0.241
\\
$f$-DRO
& 0.551
& 0.522
& 0.161
& \underline{0.217}
\\
Ours ($\epsilon=0.05$)
& \underline{0.698}
& \textbf{0.686}
& \underline{0.152}
& \textbf{0.199}
\\
\bottomrule
\end{tabularx}

\end{table}

\subsubsection{Simulation: Higher dimension, smaller samples, and heavier censoring}
\label{appendix:challenging_simulation}
We evaluate the methods with $d=10$ covariates, training sample sizes $n\in\{500,1000\}$, 5,000 test samples, and 50\% censoring.
Without contamination, the training mixture proportions of the two clean groups are $(0.9,0.1)$; with contamination, the majority, minority, and outlier proportions are $(0.90,0.08,0.02)$.
The clean test proportions remain $(0.80,0.20)$ in both settings.

We define a shared risk component and a component controlling group heterogeneity as
\[
C(x)=0.6x_1+0.5x_5,
\qquad
S(x)=1.2x_2+\frac{0.7}{\sqrt{2}}(x_3+x_4)+0.35x_6.
\]
The majority and minority log-risk functions are
\[
r_0(x)=C(x)+0.7S(x),
\qquad
r_1(x)=C(x)-0.7S(x),
\]
and the outlier log-risk function is $r_2(x)=-2(C(x)+0.7S(x))$.
The covariates $x_7,\ldots,x_{10}$ are noise variables and do not affect the true risk.
The robust subpopulation parameters are fixed at $\delta=\alpha=0.1$ for $f$-DRO and the proposed method, respectively.
For the proposed method, we compare $\epsilon\in\{0,0.01,0.02,0.05\}$ under contamination.

Tables~\ref{tab:challenging_n500} and~\ref{tab:challenging_n1000} report overall and worst-group C-index, IBS, and tail loss.
Without contamination, the robust methods improve worst-group C-index and IBS relative to Cox, with lower overall C-index and higher overall IBS.
Under contamination, increasing $\epsilon$ improves overall C-index and IBS, while intermediate values achieve the best worst-group C-index and IBS among the evaluated settings.

\begin{table}[H]
\caption{Simulation results with $d=10$, $n=500$ training samples, and 50\% censoring. Values are mean (standard deviation) over 10 runs. Bold and underlined entries denote the best and second-best results within each contamination setting, respectively.}
\label{tab:challenging_n500}
\centering
\scriptsize
\renewcommand{\arraystretch}{0.95}
\begin{tabularx}{\textwidth}{
  >{\centering\arraybackslash}l
  >{\centering\arraybackslash}X
  >{\centering\arraybackslash}X
  >{\centering\arraybackslash}X
  >{\centering\arraybackslash}X
  >{\centering\arraybackslash}X
  >{\centering\arraybackslash}X
}
\toprule
\textbf{Method ($n=500$)}
& \textbf{Overall C-index} $\uparrow$
& \textbf{Worst-group C-index} $\uparrow$
& \textbf{Overall IBS} $\downarrow$
& \textbf{Worst-group IBS} $\downarrow$
& \textbf{Overall tail loss} $\downarrow$
& \textbf{Worst-group tail loss} $\downarrow$
\\
\midrule
\multicolumn{7}{c}{\textbf{Without Outliers}} \\
\midrule
Cox
& \textbf{0.787} {\scriptsize (0.007)}
& 0.458 {\scriptsize (0.025)}
& \textbf{0.150} {\scriptsize (0.004)}
& \underline{0.357} {\scriptsize (0.022)}
& 11.315 {\scriptsize (0.373)}
& 11.850 {\scriptsize (0.581)}
\\
$f$-DRO
& \underline{0.716} {\scriptsize (0.043)}
& \textbf{0.537} {\scriptsize (0.063)}
& \underline{0.194} {\scriptsize (0.016)}
& \textbf{0.230} {\scriptsize (0.015)}
& \underline{8.711} {\scriptsize (0.130)}
& \underline{8.309} {\scriptsize (0.059)}
\\
Ours ($\epsilon=0$)
& 0.708 {\scriptsize (0.038)}
& \underline{0.526} {\scriptsize (0.049)}
& 0.207 {\scriptsize (0.012)}
& \textbf{0.230} {\scriptsize (0.008)}
& \textbf{8.568} {\scriptsize (0.068)}
& \textbf{8.245} {\scriptsize (0.019)}
\\
\midrule
\multicolumn{7}{c}{\textbf{With Outliers}} \\
\midrule
Cox
& \textbf{0.787} {\scriptsize (0.005)}
& 0.456 {\scriptsize (0.022)}
& \textbf{0.150} {\scriptsize (0.003)}
& 0.351 {\scriptsize (0.012)}
& 10.957 {\scriptsize (0.164)}
& 11.298 {\scriptsize (0.286)}
\\
$f$-DRO
& 0.628 {\scriptsize (0.044)}
& 0.469 {\scriptsize (0.088)}
& 0.212 {\scriptsize (0.012)}
& 0.246 {\scriptsize (0.016)}
& 9.030 {\scriptsize (0.148)}
& 8.694 {\scriptsize (0.244)}
\\
Ours ($\epsilon=0$)
& 0.599 {\scriptsize (0.045)}
& 0.469 {\scriptsize (0.104)}
& 0.224 {\scriptsize (0.002)}
& \underline{0.242} {\scriptsize (0.008)}
& \underline{8.844} {\scriptsize (0.154)}
& 8.397 {\scriptsize (0.025)}
\\
Ours ($\epsilon=0.01$)
& 0.732 {\scriptsize (0.019)}
& \underline{0.475} {\scriptsize (0.051)}
& 0.189 {\scriptsize (0.005)}
& 0.248 {\scriptsize (0.009)}
& 8.858 {\scriptsize (0.121)}
& \textbf{8.340} {\scriptsize (0.059)}
\\
Ours ($\epsilon=0.02$)
& 0.737 {\scriptsize (0.038)}
& \textbf{0.522} {\scriptsize (0.061)}
& 0.182 {\scriptsize (0.014)}
& \textbf{0.237} {\scriptsize (0.024)}
& \textbf{8.819} {\scriptsize (0.096)}
& \underline{8.359} {\scriptsize (0.078)}
\\
Ours ($\epsilon=0.05$)
& \underline{0.764} {\scriptsize (0.015)}
& 0.458 {\scriptsize (0.050)}
& \underline{0.169} {\scriptsize (0.010)}
& 0.287 {\scriptsize (0.026)}
& 9.275 {\scriptsize (0.248)}
& 8.574 {\scriptsize (0.153)}
\\
\bottomrule
\end{tabularx}

\end{table}

\begin{table}[H]
\caption{Simulation results with $d=10$, $n=1000$ training samples, and 50\% censoring. Values are mean (standard deviation) over 10 runs. Bold and underlined entries denote the best and second-best results within each contamination setting, respectively.}
\label{tab:challenging_n1000}
\centering
\scriptsize
\renewcommand{\arraystretch}{0.95}
\begin{tabularx}{\textwidth}{
  >{\centering\arraybackslash}l
  >{\centering\arraybackslash}X
  >{\centering\arraybackslash}X
  >{\centering\arraybackslash}X
  >{\centering\arraybackslash}X
  >{\centering\arraybackslash}X
  >{\centering\arraybackslash}X
}
\toprule
\textbf{Method ($n=1000$)}
& \textbf{Overall C-index} $\uparrow$
& \textbf{Worst-group C-index} $\uparrow$
& \textbf{Overall IBS} $\downarrow$
& \textbf{Worst-group IBS} $\downarrow$
& \textbf{Overall tail loss} $\downarrow$
& \textbf{Worst-group tail loss} $\downarrow$
\\
\midrule
\multicolumn{7}{c}{\textbf{Without Outliers}} \\
\midrule
Cox
& \textbf{0.787} {\scriptsize (0.007)}
& 0.454 {\scriptsize (0.017)}
& \textbf{0.148} {\scriptsize (0.005)}
& 0.344 {\scriptsize (0.013)}
& 10.877 {\scriptsize (0.169)}
& 11.169 {\scriptsize (0.233)}
\\
$f$-DRO
& \underline{0.701} {\scriptsize (0.061)}
& \underline{0.523} {\scriptsize (0.081)}
& \underline{0.210} {\scriptsize (0.010)}
& \underline{0.228} {\scriptsize (0.011)}
& \underline{8.722} {\scriptsize (0.102)}
& \underline{8.425} {\scriptsize (0.135)}
\\
Ours ($\epsilon=0$)
& 0.693 {\scriptsize (0.047)}
& \textbf{0.525} {\scriptsize (0.089)}
& 0.217 {\scriptsize (0.005)}
& \textbf{0.227} {\scriptsize (0.013)}
& \textbf{8.535} {\scriptsize (0.075)}
& \textbf{8.264} {\scriptsize (0.025)}
\\
\midrule
\multicolumn{7}{c}{\textbf{With Outliers}} \\
\midrule
Cox
& \textbf{0.782} {\scriptsize (0.008)}
& 0.449 {\scriptsize (0.024)}
& \textbf{0.153} {\scriptsize (0.008)}
& 0.340 {\scriptsize (0.028)}
& 10.464 {\scriptsize (0.368)}
& 10.468 {\scriptsize (0.577)}
\\
$f$-DRO
& 0.557 {\scriptsize (0.091)}
& 0.422 {\scriptsize (0.077)}
& 0.224 {\scriptsize (0.019)}
& 0.246 {\scriptsize (0.017)}
& 8.827 {\scriptsize (0.168)}
& 8.545 {\scriptsize (0.227)}
\\
Ours ($\epsilon=0$)
& 0.592 {\scriptsize (0.104)}
& 0.441 {\scriptsize (0.083)}
& 0.223 {\scriptsize (0.008)}
& \underline{0.238} {\scriptsize (0.008)}
& \textbf{8.503} {\scriptsize (0.011)}
& 8.270 {\scriptsize (0.021)}
\\
Ours ($\epsilon=0.01$)
& 0.716 {\scriptsize (0.050)}
& \textbf{0.522} {\scriptsize (0.060)}
& 0.202 {\scriptsize (0.008)}
& \textbf{0.227} {\scriptsize (0.008)}
& \underline{8.593} {\scriptsize (0.013)}
& \underline{8.249} {\scriptsize (0.078)}
\\
Ours ($\epsilon=0.02$)
& 0.746 {\scriptsize (0.025)}
& 0.455 {\scriptsize (0.063)}
& 0.188 {\scriptsize (0.008)}
& 0.244 {\scriptsize (0.016)}
& 8.755 {\scriptsize (0.093)}
& \textbf{8.224} {\scriptsize (0.017)}
\\
Ours ($\epsilon=0.05$)
& \underline{0.774} {\scriptsize (0.009)}
& \underline{0.460} {\scriptsize (0.045)}
& \underline{0.166} {\scriptsize (0.006)}
& 0.271 {\scriptsize (0.020)}
& 9.135 {\scriptsize (0.126)}
& 8.317 {\scriptsize (0.119)}
\\
\bottomrule
\end{tabularx}

\end{table}

\subsubsection{Simulation: Integrated Brier score trajectories}
\label{appendix:ibs_trajectories}
Figures~\ref{fig:linear_ibs_trajectory_summary} and~\ref{fig:nonlinear_ibs_trajectory_summary} report the integrated Brier score (IBS) across training for the linear and nonlinear simulations, respectively.
The settings follow Section~\ref{ssec:simulation} and Appendix~\ref{appendix:nonlinear_simulation}.
Worst-group IBS is the maximum IBS across the two clean test groups.

At epoch 70, Cox and Cox-NN achieve the lowest overall IBS in their respective settings, while the proposed method with $\epsilon=0.05$ achieves the lowest worst-group IBS.

\begin{figure}[H]
\centering
\includegraphics[width=\textwidth]{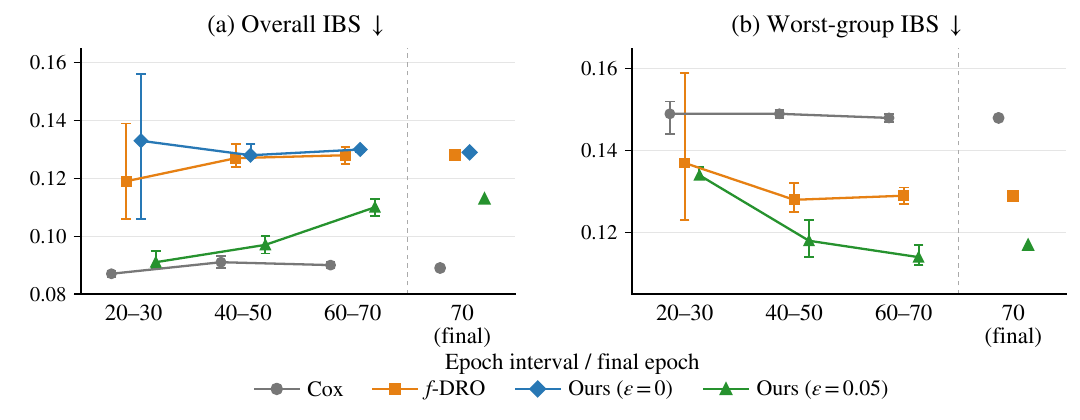}
\caption{Test IBS under linear risks and outlier contamination. Points show interval means, whiskers show minima and maxima over epochs, and lines connect means. Epoch 70 is shown separately.}
\label{fig:linear_ibs_trajectory_summary}
\end{figure}

\begin{figure}[H]
\centering
\includegraphics[width=\textwidth]{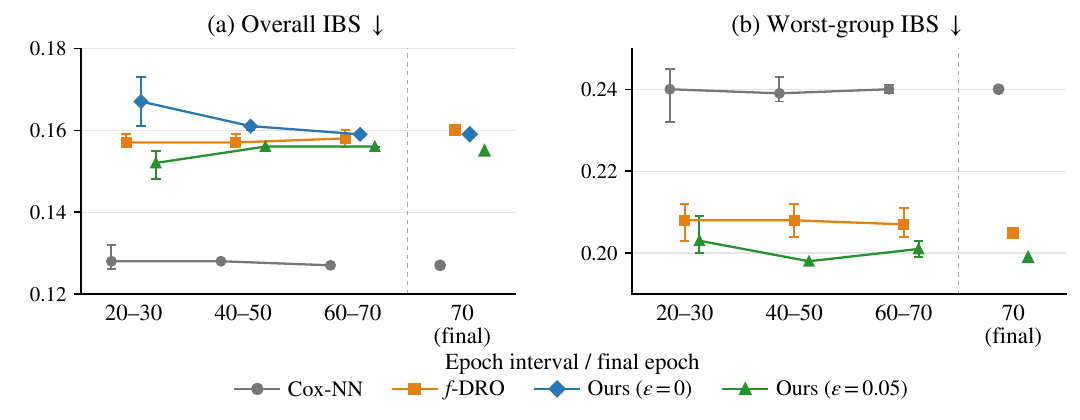}
\caption{Test IBS under nonlinear risks and outlier contamination, using the plotting conventions in Figure~\ref{fig:linear_ibs_trajectory_summary}.}
\label{fig:nonlinear_ibs_trajectory_summary}
\end{figure}

\subsubsection{Benchmark datasets and contamination setup}\label{appendix:datasets}

We describe the benchmark populations, data splits, and contamination settings used in the real-data experiments.
\paragraph{METABRIC} 
The METABRIC dataset contains breast cancer survival data with genomic, transcriptomic, and clinical information~\citep{curtis2012genomic,liu2014breast}.
We use 13 clinical covariates.
The dataset contains $1{,}793$ samples, among which $764$ observations are censored, resulting in a censoring rate of $42.61\%$.

We define four groups by estrogen receptor (ER) status and age at a threshold of $60$ years.
Group $0$ contains ER-negative patients with age $\leq 60$, Group $1$ contains ER-negative patients with age $>60$, Group $2$ contains ER-positive patients with age $\leq 60$, and Group $3$ contains ER-positive patients with age $>60$.
The group sizes are $281$, $134$, $525$, and $853$, respectively.
The corresponding censoring rates are $48.40\%$, $35.82\%$, $62.67\%$, and $29.43\%$.

\paragraph{FLC}
The FLC dataset~\citep{dispenzieri2012use} comes from a community cohort of Olmsted County, Minnesota residents aged $50$ years or older and was used to study serum free light chains and mortality.
We use $6$ covariates following \citet{hu2024fairness}.
The dataset contains $7{,}874$ samples, among which $5{,}705$ observations are censored, resulting in a censoring rate of $72.45\%$.

We define four groups by sex and age at a threshold of $60$ years.
Group $0$ consists of females with age $\leq 60$, Group $1$ consists of males with age $\leq 60$, Group $2$ consists of females with age $>60$, and Group $3$ consists of males with age $>60$.
The group sizes are $1{,}768$, $1{,}625$, $2{,}582$, and $1{,}899$, respectively.
The corresponding censoring rates are $92.82\%$, $90.09\%$, $59.80\%$, and $55.61\%$.

\paragraph{Data split}
Each dataset is split into training and test sets at an 80:20 ratio, stratified by group labels and censoring indicators.
The training set is then split into training and validation subsets at the same ratio, stratified only by censoring indicators.
Group labels are used only for the initial split and worst-group evaluation, not for training or model selection.

\paragraph{Outlier construction}
Outliers are introduced only into the training data; validation and test sets remain clean.

For linear models, we use a risk-based adversarial construction. We first train a standard Cox model on the training data and compute sample-wise risk scores
\begin{equation}
r_i = f_{\hat\theta_{\mathrm{Cox}}}(x_i).
\end{equation}
We then select the top 2\% highest-risk training samples and construct perturbed risks by sign inversion,
\begin{equation}
\widetilde{r}_i = - r_i.
\end{equation}
We generate replacement survival times from the perturbed hazard model
\begin{equation}
\widetilde{T}_i \sim \widetilde{P}(\cdot \mid x_i),
\qquad
\widetilde{h}(t \mid x_i) = h_0(t)\exp(\widetilde{r}_i).
\end{equation}
The generated times replace the selected samples' original times, creating outcomes that conflict with the fitted risk ordering.
The Cox model is fitted with an elastic-net penalty ($0.01$, mixing ratio $0.5$) on standardized covariates, so the risk scores are centered; the ranking includes censored and uncensored samples, and the event indicators of the selected samples are unchanged.
The baseline hazard is constant, $h_0(t)=\lambda_0$ with $\lambda_0=\sum_i\delta_i/\sum_i y_ie^{r_i}$, the maximum-likelihood rate given the fitted scores.

For nonlinear models, we use the same risk-based construction with the top 5\% highest-risk training samples on both datasets.

\subsubsection{Baselines}\label{appendix:baselines}
We compare methods that require no group labels for training or model selection.

\begin{itemize}[left=1pt]
    \item \textbf{Elastic-Cox}: 
    A Cox proportional hazards model with elastic-net regularization. 
    We tune the elastic-net regularization strength over 
    $\{0, 0.1, 0.2, 0.4, 0.8\}$.

    \item \textbf{Cox-NN}: 
    A neural-network Cox model with one hidden layer and hidden dimension 2. 
    We tune the Cox-NN regularization strength over 
    $\{0, 0.1, 0.2, 0.4, 0.8\}$.

    \item \textbf{$f$-DRO} \citep{hu2024fairness}: 
    A survival DRO method based on the $\chi^2$-divergence and sample splitting. 
    We tune the DRO robustness parameter over 
    $\{0.1, 0.2, 0.4, 0.8\}$.

    \item \textbf{Exact-Cox} \citep{hu2024fairness}:
    The exact DRO-Cox variant, which reparameterizes the Cox model with a piecewise-constant baseline hazard so that the full likelihood decomposes across observations, and applies the $\chi^2$-DRO objective to the resulting per-observation losses without sample splitting.
    This construction is specific to the Cox model.
    We tune the DRO robustness parameter over $\{0.1, 0.2, 0.4, 0.8\}$, initialize the baseline-hazard parameters at their maximum-likelihood values for a zero risk score, and use the same training budget and validation criterion as $f$-DRO.

    \item \textbf{Robust-Cox} \citep{minder1996robust}:
    A robust proportional hazards estimator for outlier contamination, applied to linear and nonlinear Cox models. We use the R package \texttt{coxrobust} for linear models.
    For nonlinear models, we implement an extension of the original estimator by replacing its linear predictor with the same one-hidden-layer network. Starting from a standard neural Cox fit, we optimize the truncated-weight partial likelihood over three refinement rounds.
    We tune its truncation parameter over $\{0.90,0.95,0.975,0.99\}$ and hold the remaining settings fixed.
\end{itemize}

\subsubsection{Training and hyperparameter selection}
\label{appendix:training_selection}
\paragraph{Neural-network training.}
Cox-NN, $f$-DRO, Exact-Cox, and the proposed method share a one-hidden-layer network of width 2 (ReLU hidden layer, linear output), except in the network configurations evaluated in Table~\ref{tab:network_width}, and all are trained with Adam for 100 full-batch parameter updates ($K=100$ outer iterations for the proposed method) and evaluated at the final iterate.
They differ only in the step size ($0.01$ for the baselines, $0.1$ for the proposed method) and in regularization: Cox-NN tunes an elastic-net penalty over $\{0,0.1,0.2,0.4,0.8\}$, $f$-DRO and Exact-Cox use none, and the proposed method fixes it at $0.01$.
For linear models, $f$-DRO and Exact-Cox run 100 Adam steps with step size $0.05$, whereas the proposed method runs $K=200$ gradient steps with step size $0.1$ and a fixed elastic-net penalty of $0.1$.
The nonlinear simulation uses the architecture in Appendix~\ref{appendix:nonlinear_simulation}.

\paragraph{Optimization details of the proposed method.}
Algorithm~\ref{alg:alt_nested_dro} uses $T=1$ inner weight update per outer iteration and $K=200$ outer iterations.
Step sizes are $\eta_q=0.01$, $\eta_{q'}=0.1$, and $\eta_\theta=0.1$; the $\theta$-step is plain gradient descent for the linear model and Adam for the nonlinear model.
An elastic-net penalty (strength $0.1$ linear, $0.01$ nonlinear, mixing ratio $0.5$) is added to the weighted objective.
Projections onto $\mathcal Q_{1-\epsilon}(\bar p)$ and $\mathcal Q_{\alpha}(q')$ are KL projections onto the capped simplex, $[\Pi(x)]_i=\min\{c_i,\lambda x_i\}$ with $c$ the coordinate-wise upper bounds and $\lambda$ found by bisection so that the result sums to one.

\paragraph{Validation criteria.}
Validation and test losses are Cox negative partial log-likelihoods.
Elastic-Cox and Cox-NN are selected by maximizing validation log-likelihood; Robust-Cox and the DRO methods by minimizing validation tail loss, the average of the largest 5\% or 10\% of validation losses~\citep{duchi2021learning}.
The grids for the proposed method are $\alpha\in\{0.1,0.2,0.4,0.6,0.8\}$ and $\epsilon\in\{0,0.05,0.1\}$; simulation grids are given in the corresponding subsections.

\paragraph{Practical choice of search ranges.}
Prior knowledge of the smallest meaningful subgroup and of the expected contamination proportion can guide the ranges for $\alpha$ and $\epsilon$; otherwise clustering, anomaly detection, or data-quality diagnostics provide heuristic ranges, and the final pair is selected by validation tail loss.

\paragraph{IBS evaluation.}
We compute IBS from Breslow survival curves with inverse-probability-of-censoring weights estimated by Kaplan--Meier on the training set.
Integration uses 80 equally spaced points from the smallest observed time to the 90th percentile of test times, capped at the maximum training time; test samples beyond training follow-up are excluded only from IBS evaluation.

\subsubsection{Real data analysis: Subpopulation shift}\label{appendix:real_subpopulation_shift}

Table~\ref{tab:subpop_shift} reports results under subpopulation shift without outliers.

\begin{table}[H]
\caption{Results on METABRIC and FLC under subpopulation shift without outliers. The DRO-based methods use the 10\% validation tail loss for model selection; Elastic-Cox and Cox-NN use validation log-likelihood. Values are mean (standard deviation) over 10 independent runs.}
\label{tab:subpop_shift}
\centering
\tiny
\setlength{\tabcolsep}{8.35pt}
\renewcommand{\arraystretch}{1.0}
\begin{tabular}{llccccc}
\toprule
\toprule
\multicolumn{7}{c}{\textbf{METABRIC}} \\
\midrule
\textbf{Model} & \textbf{Method}
& \textbf{IBS} $\downarrow$
& \textbf{Test log-likelihood} $\uparrow$
& \textbf{Overall C-index} $\uparrow$
& \textbf{Worst-group C-index} $\uparrow$
& \textbf{Test tail loss} $\downarrow$ \\
\midrule
\multirow{4}{*}{Linear}
& Elastic-Cox
& \textbf{0.167} (0.004)
& \textbf{-5.003} (0.037)
& 0.664 (0.014)
& 0.575 (0.047)
& 6.432 (0.162) \\
& $f$-DRO
& 0.183 (0.005)
& -5.117 (0.033)
& \underline{0.671} (0.014)
& \underline{0.583} (0.034)
& \textbf{5.829} (0.017) \\
& Exact-Cox
& \underline{0.173} (0.007)
& \underline{-5.044} (0.064)
& 0.656 (0.034)
& 0.550 (0.053)
& \underline{6.109} (0.085) \\
& Ours
& 0.186 (0.004)
& -5.128 (0.034)
& \textbf{0.683} (0.015)
& \textbf{0.607} (0.029)
& \textbf{5.829} (0.009) \\
\midrule
\multirow{4}{*}{Nonlinear}
& Cox-NN
& \textbf{0.167} (0.005)
& \textbf{-5.004} (0.043)
& \textbf{0.670} (0.015)
& \textbf{0.584} (0.027)
& 6.537 (0.103) \\
& $f$-DRO
& 0.182 (0.005)
& -5.106 (0.032)
& \underline{0.666} (0.028)
& 0.574 (0.041)
& \textbf{5.841} (0.054) \\
& Exact-Cox
& \underline{0.179} (0.006)
& -5.102 (0.040)
& 0.627 (0.034)
& 0.525 (0.083)
& 6.056 (0.053) \\
& Ours
& \underline{0.179} (0.004)
& \underline{-5.078} (0.028)
& 0.655 (0.026)
& \underline{0.575} (0.042)
& \underline{5.912} (0.034) \\
\midrule
\midrule
\multicolumn{7}{c}{\textbf{FLC}} \\
\midrule
\textbf{Model} & \textbf{Method}
& \textbf{IBS} $\downarrow$
& \textbf{Test log-likelihood} $\uparrow$
& \textbf{Overall C-index} $\uparrow$
& \textbf{Worst-group C-index} $\uparrow$
& \textbf{Test tail loss} $\downarrow$ \\
\midrule
\multirow{4}{*}{Linear}
& Elastic-Cox
& \textbf{0.093} (0.004)
& \textbf{-6.467} (0.071)
& \underline{0.797} (0.004)
& \underline{0.629} (0.039)
& 8.087 (0.262) \\
& $f$-DRO
& \underline{0.111} (0.012)
& \underline{-6.801} (0.200)
& \textbf{0.798} (0.005)
& \underline{0.629} (0.038)
& 7.400 (0.091) \\
& Exact-Cox
& 0.115 (0.007)
& -6.875 (0.112)
& 0.792 (0.007)
& 0.620 (0.038)
& \underline{7.323} (0.065) \\
& Ours
& 0.120 (0.002)
& -6.956 (0.022)
& \textbf{0.798} (0.004)
& \textbf{0.631} (0.036)
& \textbf{7.283} (0.006) \\
\midrule
\multirow{4}{*}{Nonlinear}
& Cox-NN
& \textbf{0.092} (0.001)
& \textbf{-6.452} (0.028)
& \textbf{0.797} (0.005)
& \underline{0.608} (0.026)
& 8.268 (0.036) \\
& $f$-DRO
& \underline{0.109} (0.011)
& \underline{-6.765} (0.183)
& \textbf{0.797} (0.006)
& \textbf{0.623} (0.032)
& 7.434 (0.097) \\
& Exact-Cox
& 0.116 (0.003)
& -6.890 (0.041)
& 0.774 (0.021)
& 0.556 (0.078)
& \textbf{7.330} (0.024) \\
& Ours
& 0.114 (0.003)
& -6.839 (0.048)
& \underline{0.780} (0.028)
& 0.605 (0.051)
& \underline{7.357} (0.061) \\
\bottomrule
\bottomrule
\end{tabular}
\end{table}

The proposed method reduces test tail loss relative to the standard Cox models in all settings, while these baselines achieve better IBS and test log-likelihood.
It attains the highest overall and worst-group C-index on METABRIC with linear models; on FLC, its worst-group C-index is within $0.003$ of the standard Cox models.

\begin{samepage}
\paragraph{Rationale for controlled contamination.}
\citet{luo2022penalized} model survival outliers through individual changes in hazard rates, allowing survival times to be unusually long or short relative to predictions based on covariates.
Our constructions introduce such departures at a known contamination proportion while preserving the covariates and clinically defined evaluation groups; their effects can differ across datasets and model classes.
All methods use the same contamination construction within each setting.
Appendix~\ref{appendix:additional_outliers} evaluates an alternative that appends perturbed observations while preserving the original training samples.

\subsubsection{Additional outlier setting}\label{appendix:additional_outliers}

We evaluate linear models on FLC after appending perturbed copies of short-survival event samples, preserving the original training data.
\end{samepage}
Let
\begin{equation}
\mathcal{E}
=
\{i : z_i \in \mathcal{D}_{\mathrm{train}},\ \delta_i = 1\}
\label{eq:event_set}
\end{equation}
denote the index set of uncensored event samples in the training data. 
We sort the samples indexed by $\mathcal{E}$ by their observed survival times $y_i$ in ascending order and select the lowest $2\%$.
Specifically, we define
\begin{equation}
\mathcal{O}_{\mathrm{low}}
=
\left\{
i \in \mathcal{E} : y_i \le c_{0.02}
\right\}.
\label{eq:low_event_set}
\end{equation}
Here, $c_{0.02}$ is the empirical 2nd percentile of the observed event times among samples in $\mathcal{E}$.

Each selected sample $i\in\mathcal O_{\mathrm{low}}$ is duplicated with its covariates and event indicator unchanged, and its survival time resampled near the median event time:
\begin{equation}
\widetilde{x}_i = x_i,
\qquad
\widetilde{\delta}_i = \delta_i,
\qquad
\widetilde{y}_i \sim \mathcal{U}(m_E-\tau, m_E+\tau),
\label{eq:appended_outlier_time}
\end{equation}
where $m_E$ is the median observed survival time among event samples and $\tau = 0.1 m_E$. 
The resulting duplicated sample is denoted by
\begin{equation}
\widetilde{z}_i
=
(\widetilde{x}_i, \widetilde{y}_i, \widetilde{\delta}_i).
\label{eq:appended_outlier_sample}
\end{equation}

The contaminated training set is
\begin{equation}
\widetilde{\mathcal{D}}_{\mathrm{train}}
=
\mathcal{D}_{\mathrm{train}}
\cup
\{\widetilde{z}_i : i \in \mathcal{O}_{\mathrm{low}}\}.
\label{eq:appended_contaminated_train}
\end{equation}
The appended samples retain covariates from short-survival cases but have times near the median, creating inconsistencies between covariates and outcomes.

\begin{table}[H]
\caption{Linear-model results on FLC with appended outliers. The 10\% and 5\% settings specify the validation tail loss for the DRO methods; Elastic-Cox uses validation log-likelihood. Values are mean (standard deviation) over 10 independent runs.}
\label{tab:flc_additional_outlier_linear}
\centering
\tiny
\setlength{\tabcolsep}{7.8pt}
\renewcommand{\arraystretch}{1.0}
\begin{tabular}{llccccc}
\toprule
\toprule
\multicolumn{7}{c}{\textbf{FLC}} \\
\midrule
\textbf{Model} & \textbf{Method}
& \textbf{IBS} $\downarrow$
& \textbf{Test log-likelihood} $\uparrow$
& \textbf{Overall C-index} $\uparrow$
& \textbf{Worst-group C-index} $\uparrow$
& \textbf{Test tail loss} $\downarrow$ \\
\midrule
\multirow{3}{*}{Linear (10\%)}
& Elastic-Cox
& \textbf{0.091} (0.001)
& \textbf{-6.430} (0.025)
& \textbf{0.798} (0.004)
& \underline{0.618} (0.030)
& \underline{7.520} (0.017) \\
& $f$-DRO
& 0.115 (0.016)
& -7.377 (1.157)
& 0.728 (0.106)
& 0.579 (0.092)
& 8.933 (2.887) \\
& Ours
& \underline{0.113} (0.002)
& \underline{-6.842} (0.031)
& \underline{0.797} (0.004)
& \textbf{0.644} (0.028)
& \textbf{7.180} (0.009) \\
\midrule
\multirow{3}{*}{Linear (5\%)}
& Elastic-Cox
& \textbf{0.091} (0.001)
& \textbf{-6.430} (0.025)
& \textbf{0.798} (0.004)
& \underline{0.618} (0.030)
& \underline{7.520} (0.017) \\
& $f$-DRO
& 0.129 (0.006)
& -7.629 (0.011)
& 0.688 (0.089)
& 0.563 (0.087)
& 10.177 (4.925) \\
& Ours
& \underline{0.120} (0.002)
& \underline{-6.970} (0.018)
& \underline{0.797} (0.005)
& \textbf{0.646} (0.031)
& \textbf{7.283} (0.004) \\
\bottomrule
\bottomrule
\end{tabular}
\end{table}

In Table~\ref{tab:flc_additional_outlier_linear}, Elastic-Cox achieves the best IBS, test log-likelihood, and overall C-index.
The proposed method achieves the best worst-group C-index and test tail loss under both validation criteria, with overall C-index close to Elastic-Cox.
For $f$-DRO, the 5\% validation criterion yields lower overall C-index and higher test tail loss than the 10\% criterion.

\subsubsection{Contamination sensitivity}
\label{appendix:contamination_sensitivity}
Tables~\ref{tab:outlier_shift_metabric} and~\ref{tab:outlier_shift} use the same risk-inversion construction (Appendix~\ref{appendix:datasets}) on both datasets, with 2\% contamination for linear models and 5\% for nonlinear models.
To examine sensitivity to the contamination rate, Table~\ref{tab:contamination_sensitivity} evaluates all four model--dataset combinations at both rates, using the 10\% validation setting of Table~\ref{tab:outlier_shift_metabric}.
The (linear, 2\%) and (nonlinear, 5\%) entries coincide with the 10\% validation results in Tables~\ref{tab:outlier_shift_metabric} and~\ref{tab:outlier_shift}.
At 5\%, the proposed method attains the highest worst-group C-index in all four settings.
At 2\%, it leads in the linear settings but trails the best nonlinear baselines.
Its test tail loss is the lowest or second lowest in every setting.

\begin{table}[H]
\caption{Worst-group C-index and 5\% test tail loss under the risk-based contamination applied to all model classes and datasets at contamination rates of 2\% and 5\%. Hyperparameters are selected by the 10\% validation tail loss for Robust-Cox, $f$-DRO, Exact-Cox, and the proposed method and by validation log-likelihood for Elastic-Cox and Cox-NN. Values are mean (standard deviation).}
\label{tab:contamination_sensitivity}
\centering
\tiny
\setlength{\tabcolsep}{9pt}
\renewcommand{\arraystretch}{1.0}
\begin{tabular}{llcccc}
\toprule
\toprule
& & \multicolumn{2}{c}{\textbf{2\% contamination}} & \multicolumn{2}{c}{\textbf{5\% contamination}} \\
\cmidrule(lr){3-4} \cmidrule(lr){5-6}
\textbf{Model} & \textbf{Method}
& \textbf{Worst-group C-index} $\uparrow$ & \textbf{Test tail loss} $\downarrow$
& \textbf{Worst-group C-index} $\uparrow$ & \textbf{Test tail loss} $\downarrow$ \\
\midrule
\multicolumn{6}{c}{\textbf{METABRIC}} \\
\midrule
\multirow{5}{*}{Linear}
& Elastic-Cox & 0.501 (0.058) & 6.305 (0.111) & 0.466 (0.064) & 6.195 (0.178) \\
& Robust-Cox & 0.562 (0.042) & 6.229 (0.048) & 0.520 (0.052) & 6.117 (0.091) \\
& $f$-DRO & \underline{0.574} (0.030) & \textbf{5.825} (0.015) & \underline{0.567} (0.031) & \textbf{5.822} (0.015) \\
& Exact-Cox & 0.459 (0.057) & 6.179 (0.164) & 0.393 (0.032) & 6.267 (0.176) \\
& Ours & \textbf{0.613} (0.028) & \underline{5.835} (0.011) & \textbf{0.609} (0.028) & \underline{5.835} (0.008) \\
\midrule
\multirow{5}{*}{Nonlinear}
& Cox-NN & \underline{0.572} (0.038) & 6.404 (0.084) & 0.508 (0.091) & 6.561 (0.310) \\
& Robust-Cox & \textbf{0.586} (0.040) & 6.351 (0.077) & \underline{0.545} (0.049) & 6.295 (0.161) \\
& $f$-DRO & 0.559 (0.043) & \textbf{5.832} (0.049) & 0.533 (0.047) & \textbf{5.839} (0.041) \\
& Exact-Cox & 0.419 (0.105) & 6.139 (0.100) & 0.395 (0.078) & 6.226 (0.178) \\
& Ours & 0.571 (0.039) & \underline{5.901} (0.029) & \textbf{0.585} (0.040) & \underline{5.902} (0.014) \\
\midrule
\midrule
\multicolumn{6}{c}{\textbf{FLC}} \\
\midrule
\multirow{5}{*}{Linear}
& Elastic-Cox & 0.618 (0.038) & 7.650 (0.114) & 0.602 (0.035) & 7.560 (0.128) \\
& Robust-Cox & 0.611 (0.040) & 7.854 (0.031) & 0.594 (0.040) & 7.626 (0.030) \\
& $f$-DRO & \underline{0.626} (0.037) & \underline{7.409} (0.035) & \underline{0.618} (0.035) & \underline{7.384} (0.048) \\
& Exact-Cox & 0.609 (0.037) & 7.581 (0.127) & 0.584 (0.044) & 7.506 (0.173) \\
& Ours & \textbf{0.631} (0.036) & \textbf{7.283} (0.006) & \textbf{0.632} (0.036) & \textbf{7.284} (0.006) \\
\midrule
\multirow{5}{*}{Nonlinear}
& Cox-NN & \underline{0.614} (0.037) & 7.917 (0.183) & 0.623 (0.029) & 7.801 (0.208) \\
& Robust-Cox & \textbf{0.626} (0.041) & 8.071 (0.173) & 0.602 (0.030) & 8.134 (0.324) \\
& $f$-DRO & \textbf{0.626} (0.031) & \underline{7.457} (0.035) & \underline{0.624} (0.028) & \underline{7.447} (0.043) \\
& Exact-Cox & 0.582 (0.057) & 7.603 (0.209) & 0.561 (0.081) & 7.622 (0.188) \\
& Ours & 0.613 (0.040) & \textbf{7.359} (0.064) & \textbf{0.630} (0.033) & \textbf{7.371} (0.067) \\
\bottomrule
\bottomrule
\end{tabular}
\end{table}

\subsubsection{Computational runtime}
\label{appendix:runtime}
Table~\ref{tab:runtime_comparison} compares model-fitting times for linear Cox models on METABRIC and FLC.
For this comparison, all optimization-based methods, including the proposed method, are run for 100 iterations, separately from the training schedules in Appendix~\ref{appendix:training_selection}.
The proposed method is faster than $f$-DRO on both datasets, although standard Cox and Exact-Cox are faster.
With the $K=200$ schedule of Appendix~\ref{appendix:training_selection}, it takes $0.631$ and $1.283$ seconds on METABRIC and FLC, respectively.
The corresponding complexity analysis is given in Appendix~\ref{appendix:computational_complexity}.

\begin{table}[H]
\caption{Model-fitting time in seconds, reported as mean (standard deviation) over 10 repetitions on a single CPU thread. All optimization-based methods are run for 100 iterations in this comparison.}
\label{tab:runtime_comparison}
\centering
\footnotesize
\setlength{\tabcolsep}{4pt}
\renewcommand{\arraystretch}{0.95}
\begin{tabularx}{\textwidth}{
  >{\centering\arraybackslash}l
  >{\centering\arraybackslash}X
  >{\centering\arraybackslash}X
  >{\centering\arraybackslash}X
  >{\centering\arraybackslash}X
  >{\centering\arraybackslash}X
}
\toprule
\textbf{Dataset}
& \textbf{Cox}
& \textbf{Elastic-Cox}
& \textbf{$f$-DRO}
& \textbf{Exact-Cox}
& \textbf{Ours}
\\
\midrule
METABRIC
& 0.104 {\scriptsize (0.005)}
& 1.239 {\scriptsize (0.040)}
& 1.129 {\scriptsize (0.281)}
& 0.195 {\scriptsize (0.004)}
& 0.368 {\scriptsize (0.008)}
\\
FLC
& 0.356 {\scriptsize (0.133)}
& 2.210 {\scriptsize (0.796)}
& 18.122 {\scriptsize (0.531)}
& 0.236 {\scriptsize (0.007)}
& 0.817 {\scriptsize (0.135)}
\\
\bottomrule
\end{tabularx}

\end{table}


\end{document}